\documentclass[letterpaper]{article} %
\usepackage{aaai25}  %
\usepackage{times}  %
\usepackage{helvet}  %
\usepackage{courier}  %
\usepackage[hyphens]{url}  %
\usepackage{graphicx} %
\usepackage{natbib}  %
\usepackage{caption} %
\usepackage{algorithm}
\usepackage{algorithmic}
\usepackage[T1]{fontenc}
\usepackage{xcolor}
\usepackage{booktabs}
\usepackage{adjustbox}
\usepackage{multirow}
\usepackage{subcaption}
\usepackage{url}
\usepackage{mathtools}
\usepackage{yhmath}
\usepackage{dsfont}
\usepackage{amsthm}
\usepackage{amsmath}
\usepackage{amssymb}
\usepackage{pgfplots}

\pgfplotsset{compat=1.18}
\usepackage{amsfonts}

\renewcommand{\hat}{\widehat}

\newcommand{\Itrain}{\mathcal I_{\textrm{train}}}
\newcommand{\Ical}{\mathcal I_{\textrm{cal}}}
\newcommand{\Itest}{\mathcal I_{\textrm{test}}}

\newtheorem{prop}{Proposition}

\definecolor{matblue}{HTML}{1F77B4}
\newtheorem{theorem}{Theorem}
\newtheorem*{remark}{Remark}
\newtheorem{lemma}{Lemma}

\newcounter{protocol}%
\newcounter{algorithm saved}%



\usepackage{newfloat}
\usepackage{listings}
\DeclareCaptionStyle{ruled}{labelfont=normalfont,labelsep=colon,strut=off} %
\floatstyle{ruled}
\newfloat{listing}{tb}{lst}{}
\floatname{listing}{Listing}
\title{Conformal Thresholded Intervals for Efficient Regression}
\author {
    Rui Luo\textsuperscript{\rm 1}\thanks{Corresponding author}  and
    Zhixin Zhou\textsuperscript{\rm 2}
}
\affiliations {
    \textsuperscript{\rm 1}Department of Systems Engineering, City University of Hong Kong, Hong Kong SAR, China\\
    \textsuperscript{\rm 2}Alpha Benito Research, Los Angeles, USA\\
    ruiluo@cityu.edu.hk, zzhou@alphabenito.com
}

\usepackage{bibentry}

\begin{document}

\maketitle

\begin{abstract}
This paper introduces Conformal Thresholded Intervals (CTI), a novel conformal regression method that aims to produce the smallest possible prediction set with guaranteed coverage. Unlike existing methods that rely on nested conformal frameworks and full conditional distribution estimation, CTI estimates the conditional probability density for a new response to fall into each interquantile interval using off-the-shelf multi-output quantile regression. By leveraging the inverse relationship between interval length and probability density, CTI constructs prediction sets by thresholding the estimated conditional interquantile intervals based on their length. The optimal threshold is determined using a calibration set to ensure marginal coverage, effectively balancing the trade-off between prediction set size and coverage. CTI's approach is computationally efficient and avoids the complexity of estimating the full conditional distribution. The method is theoretically grounded, with provable guarantees for marginal coverage and achieving the smallest prediction size given by Neyman-Pearson . Extensive experimental results demonstrate that CTI achieves superior performance compared to state-of-the-art conformal regression methods across various datasets, consistently producing smaller prediction sets while maintaining the desired coverage level. The proposed method offers a simple yet effective solution for reliable uncertainty quantification in regression tasks, making it an attractive choice for practitioners seeking accurate and efficient conformal prediction.
\end{abstract}

\begin{links}
    \link{Code}{https://github.com/luo-lorry/CTI}
    \link{Extended version}{https://arxiv.org/abs/2407.14495}
\end{links}

\section{Introduction}
Conformal prediction is a powerful framework for constructing prediction sets with finite-sample coverage guarantees. By leveraging exchangeability of the data, conformal methods can convert the output of any machine learning algorithm into a set-valued prediction satisfying the required coverage level, without assumptions on the data distribution. This paper develops a novel conformal prediction method for regression that aims to produce the smallest possible prediction set with guaranteed coverage.

Most existing conformal methods for regression either directly predict the lower and upper endpoints of the interval using quantile regression models \cite{romano2019conformalized, kivaranovic2020adaptive,sesia2020comparison,gupta2022nested} or first estimate the full conditional distribution of the response and then invert it to obtain prediction sets \cite{izbicki2020flexible,chernozhukov2021distributional}. While these approaches perform well in many situations, they may produce sub-optimal prediction sets if the conditional distribution is skewed. Conformal quantile regression typically yields equal-tailed intervals, but the shortest valid interval may be unbalanced. On the other hand, density-based methods can adapt to skewness but typically involve many tuning parameters and more difficult interpretation, which can be complex for practitioners.

To address these limitations, we propose conformal thresholded intervals (CTI), a conformal inference method that seeks the smallest possible prediction set. Instead of relying on an estimate of the full conditional distribution, we use off-the-shelf multi-output quantile regression and construct prediction set by thresholding the estimated conditional interquantile intervals. Compared with conformal histogram regression (CHR) \cite{sesia2021conformal}, which first partitions the response space into bins, CTI directly trains a multi-output quantile regression model that uses equiprobable quantiles. This allows us to estimate the conditional probability density for a new response to fall into each interquantile interval, without the need for explicitly binning the response space.

For each sample in the calibration set, we obtain the interquantile interval that its response falls into and find the corresponding probability density estimate. We compute the non-conformity scores based on these estimates. Intuitively, the non-conformity score is higher for a sample that falls into a long interquantile interval and lower for a sample that falls into a short interquantile interval. By adopting a similar thresholding idea as in conformal classification \cite{sadinle2019least, luo2024trustworthy}, we threshold the intervals according to their length, the inverse of which corresponds to the probability density estimate. At test time, the threshold, i.e., the quantile for non-conformity scores, is used in constructing prediction sets for test samples. Specifically, the interquantile intervals are sorted in ascending order of length, and the first ones shorter than or equal to the threshold are kept. We show that the prediction sets generated from thresholding interquantile intervals guarantee marginal coverage and can achieve desired conditional coverage as well as the smallest expected prediction interval length if the multi-output quantile regression model produces true conditional probability density estimates.

The main contributions of this paper are as follows: 
\begin{enumerate}
    \item We propose a novel conformal prediction method for regression, CTI, that aims to produce the smallest possible prediction set with guaranteed coverage. CTI leverages multi-output quantile regression to estimate conditional interquantile intervals and constructs prediction sets by thresholding these intervals.
    \item We provide a theoretical analysis showing that CTI guarantees marginal coverage and can achieve desired conditional coverage as well as the smallest expected prediction interval length under certain conditions.
    \item We conduct extensive numerical experiments on both simulated and real datasets, demonstrating that CTI compares favorably to existing conformal regression methods in terms of prediction set size while maintaining valid coverage.
    \item The CTI method is simple to implement and interpret, making it an attractive choice for practitioners seeking reliable uncertainty quantification in regression tasks.
\end{enumerate}

The rest of this paper is organized as follows. We discuss related work in Section~\ref{sec:related:work}. Section~\ref{sec:preliminary} introduces the problem setup and competitive methods. Section~\ref{sec:cti} describes the proposed CTI method in detail. Section~\ref{sec:theory} presents a theoretical analysis of CTI. Section~\ref{sec:exp} provides numerical experiments comparing CTI to existing conformal regression methods on both simulated and real data. Finally, Section~\ref{sec:conclusion} concludes with a discussion of the main results and future directions.

\section{Related Work}
\label{sec:related:work}
Quantile regression \cite{koenker2005quantile} estimates the $\tau$-th conditional quantile function by minimizing the check function loss:
\begin{align*}
    \min_{f_{\tau}} \sum_{i=1}^{n} \rho_{\tau} (y_i - f_{\tau}(x_i)),
\end{align*}
where 
\begin{align*}
    \rho_{\tau}(r) = \begin{cases}
        \tau r, & \textrm{if $r > 0$} \\
        -(1-\tau) r, & \textrm{otherwise}
    \end{cases}
\end{align*}
is the check function representing the absolute loss. 

Quantile regression has been widely used to construct prediction intervals by estimating conditional quantile functions at specific levels, such as the 5\% and 95\% levels for 90\% nominal coverage \cite{hunter2000quantile, taylor2000quantile, meinshausen2006quantile, takeuchi2006nonparametric, steinwart2011estimating}. This approach adapts to local variability, even for highly heteroscedastic data. 

Simultaneous estimation of multiple quantiles is asymptotically more efficient than separate estimation of individual regression quantiles or ignoring within-subject dependency \cite{cho2017multiple}. However, this approach does not guarantee non-crossing quantiles, which can affect the validity of the predictions and introduce critical issues in certain scenarios. To address this limitation, research on non-crossing multiple quantile regression has gained attention in recent years, with several methods proposed to ensure non-crossing quantile estimates, including stepwise approaches \cite{liu2009stepwise}, non-parametric techniques \cite{cannon2018non}, and deep learning-based models \cite{moon2021learning, brando2022deep}.

However, the validity of the produced intervals is only guaranteed for specific models under certain regularity and asymptotic conditions \cite{steinwart2011estimating, takeuchi2006nonparametric, meinshausen2006quantile}. Many related methods for constructing valid prediction intervals can be encompassed within the nested conformal prediction framework, where a nested sequence of prediction sets is generated by thresholding nonconformity scores derived from various approaches, such as residual-based methods \cite{papadopoulos2002inductive, balasubramanian2014conformal, lei2018distribution}, quantile regression \cite{romano2019conformalized, kivaranovic2020adaptive, sesia2020comparison, chernozhukov2021distributional}, density estimation \cite{izbicki20a, sesia2021conformal, izbicki2022cd}, and their combinations with ensemble methods \cite{gupta2022nested} and localized methods \cite{papadopoulos2008normalized, colombo2023training, luo2024conformal}. However, as noted by \cite{lei2013distribution}, the optimal conditionally-valid prediction regions are level sets of conditional densities, which need not be intervals, suggesting that constructing possibly non-convex prediction sets might lead to more efficient conformal predictors.

Our proposed method for constructing non-convex prediction sets is related to the work of \cite{izbicki2022cd}, who introduce a profile distance to measure the similarity between features and construct prediction sets based on neighboring samples. 
In contrast, our method directly estimates the conditional probability density for a new response to fall into each interquantile interval based on a multi-output quantile regression model. By thresholding the interquantile intervals based on their length, which is inversely proportional to the estimated probability density, we can construct efficient prediction sets that adapt to the local density of the data. This approach allows us to generate prediction sets that are not restricted to intervals and can potentially achieve better coverage and efficiency compared to interval-based methods.

Another related approach \cite{guha2024conformal} converts regression to a classification problem and employs a conditional distribution with a smoothness-enforcing penalty. This method is orthogonal to our approach and can be potentially combined with our multi-output quantile regression framework to further improve the efficiency of the constructed prediction sets.

\section{Preliminary and Problem Setup}\label{sec:preliminary}

In this section, we first provide an overview of conformal prediction methods for regression problems in the literature. We then introduce a simple setting for studying conformal interval arithmetic, where the calibration and test sets are at a group level. In this setting, we propose using the residual of sums to perform conformal prediction for the test group.

\subsection{Conformal Prediction for Regression}\label{sec:conformal_regression}

Conformal prediction \cite{vovk2005algorithmic, shafer2008tutorial} is a framework for constructing prediction intervals with guaranteed marginal coverage in regression problems. It leverages \emph{conformity scores} to measure how well each sample fits the model and uses the empirical distribution of these scores from a calibration set to determine prediction interval sizes for new test points.

Given a dataset $\{(x_i, y_i)\}_{i=1}^n$, split into training, calibration, and test sets with index sets $\mathcal{I}_{\text{train}}, \mathcal{I}_{\text{cal}}, \mathcal{I}_{\text{test}}$, a model $\hat{f}$ is first trained on $\mathcal{I}_{\text{train}}$ to produce predictions $\hat{y}_i = \hat{f}(x_i)$. Conformity scores $S_i$ are then computed for the calibration data:
\begin{align}\label{eq:score:function}
S_i = s(x_i, y_i), \quad \forall i \in \mathcal{I}_{\text{cal}},
\end{align}
where $s$ is a score function. In the \emph{split conformal method} \cite{papadopoulos2002inductive, vovk2005algorithmic}, the score is typically the absolute residual:
\begin{align}\label{eq:split_score}
S_i^{\text{Split}} = |y_i - \hat{f}(x_i)|.
\end{align}
For a new test point $x_{n+1}$, the $1-\alpha$ prediction interval is:
\begin{align}\label{eq:split_interval}
\mathcal{C}_{n,\alpha}(x_{n+1}) = \{y \in \mathbb{R} : s(x_{n+1}, y) \leq t_{1-\alpha}^{\text{Split}}\},
\end{align}
where $t_{1-\alpha}^{\text{Split}}$ is the $(1-\alpha)(1+1/|\mathcal{I}_{\text{cal}}|)$-th empirical quantile of $\{S_i^{\text{Split}}\}_{i \in \mathcal{I}_{\text{cal}}} \cup \{\infty\}$. Under the exchangeability assumption, this guarantees marginal coverage:
\begin{align}\label{eq:marginal_coverage}
\mathbb{P}(Y_{n+1} \in \mathcal{C}(x_{n+1})) \geq 1-\alpha.
\end{align}

Two advanced methods, \emph{Conformal Quantile Regression (CQR)} \cite{romano2019conformalized} and \emph{Conformal Histogram Regression (CHR)} \cite{sesia2021conformal}, extend this framework:
\begin{itemize}
    \item \textbf{CQR} constructs intervals based on quantile regression estimates:
    \begin{align}\label{eq:cqr_interval}
    \mathcal{C}^{\text{CQR}}(x_{n+1}) = \left[\hat{q}_{\frac{\alpha}{2}}(x_{n+1}) - t_{1-\alpha}^{\text{CQR}}, \hat{q}_{1-\frac{\alpha}{2}}(x_{n+1}) + t_{1-\alpha}^{\text{CQR}}\right],
    \end{align}
    where $\hat{q}_{\frac{\alpha}{2}}$ and $\hat{q}_{1-\frac{\alpha}{2}}$ are conditional quantile estimates, and $t_{1-\alpha}^{\text{CQR}}$ is the $(1-\alpha)(1+1/|\mathcal{I}_{\text{cal}}|)$-th empirical quantile of $\{S_i^{\text{CQR}}\}_{i \in \mathcal{I}_{\text{cal}}} \cup \{\infty\}$, with:
    \begin{align}\label{eq:cqr_score}
    S_i^{\text{CQR}} = \min\big(\hat{q}_{\alpha/2}(x_i) - y_i, y_i - \hat{q}_{1-\alpha/2}(x_i)\big).
    \end{align}

    \item \textbf{CHR} constructs prediction intervals by estimating the full conditional density $f_{Y|X}$ using histograms and finding the shortest interval $(a, b)$ such that:
    \begin{align}\label{eq:chr_interval}
    \mathcal{C}^{\text{CHR}}(x_{n+1}) = \arg\min_{a < b} (b - a), \\ 
    \text{s.t.} \int_a^b \hat{f}_{Y|X}(y|x_{n+1}) \, dy \geq 1-\alpha.
    \end{align}
\end{itemize}

Both methods improve efficiency (tighter intervals) compared to the split conformal method, with CQR adapting to local variability and CHR accommodating non-standard distributions.

\subsection{Problem Setup}\label{sec:problem_setup}

We consider a regression problem with a dataset $\{(x_i, y_i)\}_{i=1}^n$, where $x_i \in \mathcal{X} \subseteq \mathbb{R}^d$ and $y_i \in \mathcal{Y} \subseteq \mathbb{R}$. The dataset is split into training, calibration, and test sets with index sets $\mathcal{I}_{\text{train}}, \mathcal{I}_{\text{cal}}, \mathcal{I}_{\text{test}}$. We assume the calibration and test samples are exchangeable.

The goal is to construct a \emph{conformal predictor} that outputs a prediction set $\mathcal{C}(x) \subseteq \mathcal{Y}$ for each test input $x$, such that the true response $y$ satisfies:
\begin{align}\label{eq:coverage_goal}
\mathbb{P}(Y \in \mathcal{C}(X)) \geq 1-\alpha,
\end{align}
where $\alpha \in (0, 1)$ is a user-specified significance level. While ensuring valid marginal coverage, we aim to minimize the expected size of the prediction sets:
\begin{align}\label{eq:minimize_size}
\mathbb{E}[\mu(\mathcal{C}(X))] = \int_\mathcal{X} \mu(\mathcal{C}(x)) \, dP(x),
\end{align}
where $\mu$ is the Lebesgue measure and $P(x)$ is the marginal distribution of $X$.

Existing methods like CQR and CHR produce interval-based prediction sets, which may be suboptimal for non-unimodal or non-symmetric conditional distributions. Our objective is to develop a more flexible approach that generates general, possibly non-convex prediction sets, improving efficiency while maintaining valid coverage.

\section{Conformal Thresholded Interquantile Intervals}
\label{sec:cti}
First, we apply quantile regression on the training set $\mathcal{D}_{\text{train}}$ to predict the $\tau$-th quantile of the conditional distribution $Y|X=x$ for every $x\in \mathcal X$, where $\tau$ takes values from 0 to 1 in increments of $1/K$. The estimated quantile for $\tau=k/K$ is denoted by 
\begin{align}\label{eq:quantile:estimation}
\hat q_k(x) \quad\text{for}\quad k=0,1, \dots, K.
\end{align}
We then define the interquantile intervals as 
\begin{align}\label{eq:interquantile:estimation}
    I_k(x) = ( \hat q_{k-1}(x), \ \hat q_{k}(x) ]\quad\text{for}\quad k= \ 1,\dots, K.
\end{align}
Assuming the quantile regression provides sufficiently accurate estimations, each interval should have approximately the same probability, $1/K$, of covering the true label $Y$. To minimize the size of the prediction set, it is more efficient to include intervals with smaller sizes. This strategy leads us to define the confidence set as:
\begin{align}\label{eq:confidence:set}
    \mathcal C(x) = \bigcup\{I_k(x): \mu(I_k(x))\le t, k=1,\dots, K\},
\end{align}
where $t$ is a threshold determined in a marginal sense, meaning it is independent of $x$. To determine $t$, we utilize the calibration set. We want $t$ to satisfy the condition that $y_i\in\mathcal C(x_i)$ for at least $\lceil (1+|\mathcal I_{\text{cal}}|)(1-\alpha)\rceil$ instances in the calibration set, where $i\in  \mathcal I_{\text{cal}}$. We define the threshold $t$ as follows:
\begin{align}\label{eq:threshold}
    \begin{split}
    t &= (1 - \alpha)\text{-th quantile of the empirical distribution} \\
    & \frac{1}{(1 + |I_\text{cal}|)} \sum_{i \in I_\text{cal}} \delta_{\mu(I_{k(y_i)}(x_i))} + \delta_{\infty}
    \end{split}
\end{align}
where $k(y)$ is the index that of the interval that $y$ belongs, i.e., $y\in I_{k(y)}(x)$. 
By plugging $t$ back into~\eqref{eq:confidence:set}, we obtain the prediction set for every $x\in\mathcal X$. 
The above procedure is summarized the following algorithm.  
\begin{algorithm}
\caption{Conformalized Thresholded Intervals}
\begin{algorithmic}[1]
\STATE \textbf{Input:} labeled data $\{(x_i, y_i)\}_{i\in\mathcal I}$, unlabel test data $\{x_i\}_{\Itest}$, a data split ratio, black-box learning algorithm $\mathcal{B}$, level $\alpha \in (0, 1)$, number of interquantile intervals $K$
\STATE Randomly split the indices $\mathcal I$ into $\Itrain$ and $\Ical$. 
\STATE Train $\mathcal{B}$ on samples in $\Itrain$, and obtain quantile estimation functions $\hat q_k$ for $k = 0, 1, \dots, K$. 
\STATE For every $i\in\Ical\cup\Itest$, evaluate $\hat q_k(x_i)$ for $k=0, 1, \dots, K$.
\STATE For every $i\in\Ical\cup\Itest$, define the interquantile intervals $I_k(x_i) = (\hat q_{k-1}(x_i), \hat q_k(x_i)]$ for $k=1, \dots, K$. 
\STATE $t\gets (1 - \alpha)\text{-th quantile of the empirical distribution} \newline
\frac{1}{(1 + |I_\text{cal}|)} \sum_{i \in I_\text{cal}} \delta_{\mu(I_{k(y_i)}(x_i))} + \delta_{\infty}$.
\STATE For $i\in\Itest, \mathcal C(x_i) = \bigcup\{I_k(x_i): \mu(I_k(x_i))\le t, k=1, \dots, K\}.$
\STATE \textbf{Output:} $\mathcal C(x_i)$ for $i\in\Itest$. 
\end{algorithmic}\label{alg:cti}
\end{algorithm}

\begin{remark}
    Our approach can also be considered in terms of conformity scores. Using the definition of the prediction set in equation~\eqref{eq:confidence:set}, the value of label $y$ is contained within a small interval. More formally, let $k(y)$ be the index such that $y\in I_{k(y)}(x)$. We can then define the score function (\ref{eq:score:function}) for our proposed method as:
    \begin{align*}
    s(x,y) = \mu(I_{k(y)}(x)).
    \end{align*}
    This score function assigns a score to each label $y$ based on the size of the interval $I_{k(y)}(x)$ in which it falls. A smaller score indicates that the label $y$ is more likely to be the true label for the input $x$. In the context of conformal prediction, labels with smaller scores are given priority for inclusion in the prediction set. 
    We present a detailed discussion of the connection of Algorithm \ref{alg:cti} and Theorem \ref{thm:coverage} in the Appendix.
\end{remark}

\section{Theoretical Analysis}\label{sec:theory}

In the context of the entire population, CTI shares a very similar formulation with the Least Ambiguous Set  method used for classification, as described in~\cite{sadinle2019least}. If we assume that our quantile regression model is sufficiently accurate, CTI has the potential to achieve the optimal size for prediction sets when considering the marginal distribution. To understand this better, let's first take a look at the Neyman-Pearson Lemma:

\begin{lemma}[Neyman-Pearson] \label{lem:np}
Let $f$ and $g$ be two nonnegative measurable functions. Then the optimizer of the problem
\begin{align*}
\min_{C} \int_{C} g \quad \text{s.t. } \quad \int_{C} f \geq 1 - \alpha,
\end{align*}
is given by $C = \{x: f(x)/g(x) \geq t'\}$ if there exists $t$ such that $\int_{f/g \geq t'} f = 1 - \alpha$.
\end{lemma}
To formalize the problem of minimizing the expected length of the prediction set subject to $1-\alpha$ coverage, we can write the problem as:
\begin{align*}
\min_{\mathcal C(x)} & \int_\mathcal X\int_{\mathcal C(x)} 1 d\mu(y) dP(x)
\\ 
& \text{s.t.} \quad \int_{\mathcal X} \int_{\mathcal C(x)} f(y|x) d\mu(y)dP(x)\ge 1-\alpha.
\end{align*}
The Neyman-Pearson Lemma implies that the optimal solution for $\mathcal C(x)$ has the form:
\begin{align}\label{eq:optimal:prediction:set}
\mathcal C(x) = \{y: f(y|x)\ge t'\}
\end{align}
for some suitable threshold $t'$. Indeed, this threshold can be defined as 
\begin{align}\label{eq:threshold:density}
    t' = \inf\{ t\in\mathbb R: \mathbb P(f(Y|X)\ge t)\ge 1-\alpha \}.
\end{align}
which is discussed in the Appendix in detail. 
Our algorithm is an empirical construction of such an interval.
Suppose the quantile regression approximates $\hat q_\tau$ well. In that case, we have:
\begin{align*}
\int_{y\in I_k(x)} f(y|x) d\mu(y)=\mathbb P(Y\in I_k(X))\approx 1/K.
\end{align*}
As $K$ approaches infinity, $\mu(I_k(x))$ tends to 0. If $f(y|x)$ is sufficiently smooth, then $f(y|x)\approx 1/(K\mu(I_k(x)))$.
The threshold on the length of intervals $\mu(I_k(x))\le t$ in equation~\eqref{eq:confidence:set} approximately implies $f(y|x)\ge 1/(Kt)$, which is optimal in the sense of the Neyman-Pearson Lemma. This means that our algorithm, which constructs prediction sets based on the threshold on interval lengths, is an empirical approximation of the optimal solution prescribed by the Neyman-Pearson Lemma.

The demonstration of the coverage probability for CTI follows the same reasoning as the traditional proof used in the general conformal prediction framework.

\begin{theorem}[Coverage Probability]\label{thm:coverage}
    Suppose the samples in $\{(X_i, Y_i)\}_{i\in\Ical\cup\Itest}\}$ are exchangeable, then for $(X,Y)$ in the test set, the coverage probability
    \begin{align*}
        \mathbb P\left(Y\in\mathcal C(X) \right)\ge 1-\alpha.
    \end{align*}
\end{theorem}

The upcoming proposition will demonstrate that the threshold $t$, which is defined in equation~\eqref{eq:threshold} for the length of interquantile intervals, results in a suitable threshold for the distribution of the conditional density $f(Y|X)$, as shown in equation~\eqref{eq:threshold:density}.

\begin{prop}[Threshold Consistency]\label{prop:consistency}
    Suppose the interquantile intervals satisfy
    \begin{align*}%
    \sup_{t'}\left| F_{f(Y|X)}(t') - \frac 1{\Ical}\sum_{i\in \Ical}\mathbf 1\big\{\mu(I_{k(y_i)}(x_i))> \frac{1}{Kt'}\big\} \right|\le \epsilon
\end{align*}
Let $t$ be as defined in~\eqref{eq:threshold} and let $1-\alpha'=\lceil (1+|\mathcal I_{\text{cal}}|)(1-\alpha)\rceil/|\Ical|$, then $F_{f(Y|X)}(1/(Kt))\ge \alpha'-\epsilon$. 
\end{prop}

This proposition establishes that if the value of $\epsilon$ is sufficiently small and the size of the calibration set is large enough, then the value $t'=1/(Kt)$, where $t$ is the threshold defined in equation~\eqref{eq:threshold}, will be close to the $\alpha$-th quantile of the distribution $f(Y|X)$.

Lastly, the theorem presented below demonstrates that the size of the prediction set obtained using CTI will not significantly exceed the size of the theoretically optimal prediction set, which is defined in equation~\eqref{eq:optimal:prediction:set}. 

\begin{theorem}[Prediction Set] \label{thm:set}
Suppose for $x\in \mathcal X$,
    \begin{align*}%
    P(Y\in I_k(x)|X=x) = \int_{y\in I_k(x)} f(y|x) dy \ge \frac {1-\delta(x)}K,
\end{align*}
and suppose $f(y|x)$ has a Lipschitz constant $L(x)$, then
\begin{align*}
    \mathcal C_{1-\alpha}(x)\subseteq \left\{y: f(y|x)\ge \frac{1-\delta_k(x)}{Kt} - \frac {L(x)t}2\right\}.
\end{align*}
\end{theorem}

The previous proposition demonstrates that the value $1/(Kt)$, where $t$ is the threshold defined in equation~\eqref{eq:threshold}, is slightly smaller than the $\alpha$-th quantile of  $f(Y|X)$, which represents the theoretically optimal threshold as shown in equation~\eqref{eq:threshold:density}. This theorem further illustrates that the prediction set obtained using CTI will include values that are even more conservative. However, we understand that $1/(Kt)$, where $t$ is the threshold defined in equation~\eqref{eq:threshold}, serves as a relatively stable estimate of the quantile, which is asymptotically equivalent to a constant value. As the number of interquantile intervals $K$ approaches infinity, the threshold $t$ converges to 0. If we additionally assume that our quantile regression model is accurate, meaning that the error term $\delta_k(x)$ is small, then the prediction set $\mathcal C_{1-\alpha}(x)$ obtained using CTI will be close to the theoretically optimal prediction set defined in equation~\eqref{eq:confidence:set}.

\begin{remark}[Comparison with existing methods]
Both CTI and CHR use multi-output quantile regression but differ in their approach to achieving coverage. CHR constructs prediction sets with an expected $1-\alpha$ coverage for each input $x$, focusing on conditional coverage. In contrast, CTI adopts a global perspective by thresholding interquantile intervals across all $x$ values. This global approach improves efficiency in achieving marginal coverage by deprioritizing $x$ values where interquantile ranges are large and uncertain. By focusing on the overall distribution, CTI allocates coverage more effectively, resulting in efficient prediction sets.

Multi-output quantile regression often requires training separate models for each quantile, which can be computationally expensive. CTI addresses this by using Quantile Regression Forests (RF) \cite{meinshausen2006quantile} and Quantile Regression Neural Networks (NN) to estimate multiple quantiles simultaneously. RF aggregates quantile estimates from an ensemble of decision trees, avoiding separate optimization for each quantile. Similarly, NN shares parameters across a single neural network to output multiple quantiles as a vector. This efficiency allows CTI to generate prediction sets tailored to user-specified tolerance levels without redundant model training.

To balance expressiveness and computational efficiency, we fix the number of quantiles at $K=100$ for all datasets. This choice ensures affordable runtime for CTI while maintaining flexibility for various tasks. We present a comprehensive discussion of the choice of $K$ in the Appendix.

As shown in Lemma \ref{lem:np} and Theorem \ref{thm:set}, the optimality of CTI is guaranteed by the Neyman-Pearson Lemma. Specifically, as the error term $\delta_k(x) \to 0$, $K \to \infty$, and $t \to 0$, the prediction set $S$ converges to
\[
S \to \{x \,|\, p(x) \geq \frac{1}{Kt}\},
\]
ensuring both validity and efficiency in prediction set construction.
\end{remark}

\section{Experiment}\label{sec:exp}

\subsection{Simulation Study}
We presented an experiment on simulated data to illustrate the importance of adaptivity in conformal prediction. We also derive the analytical quantile function at different levels and the expected prediction set sizes of various conformal methods and demonstrate the superiority of CTI.

\begin{figure}[t!]
\centering
\begin{subfigure}{0.23\textwidth}
\centering
\includegraphics[width=\textwidth]{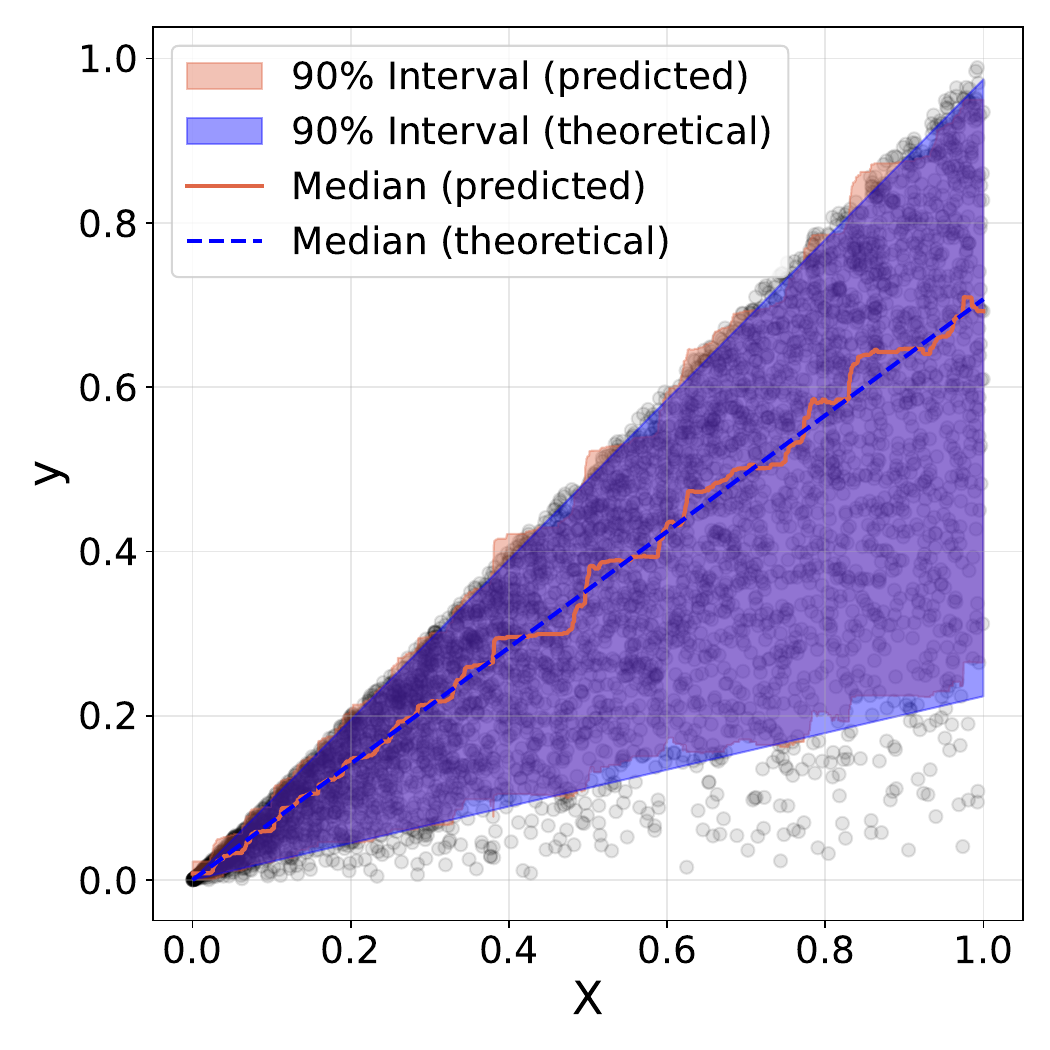}
\caption{Random Forest}
\label{fig:qrf}
\end{subfigure}
\hfill
\begin{subfigure}{0.23\textwidth}
\centering
\includegraphics[width=\textwidth]{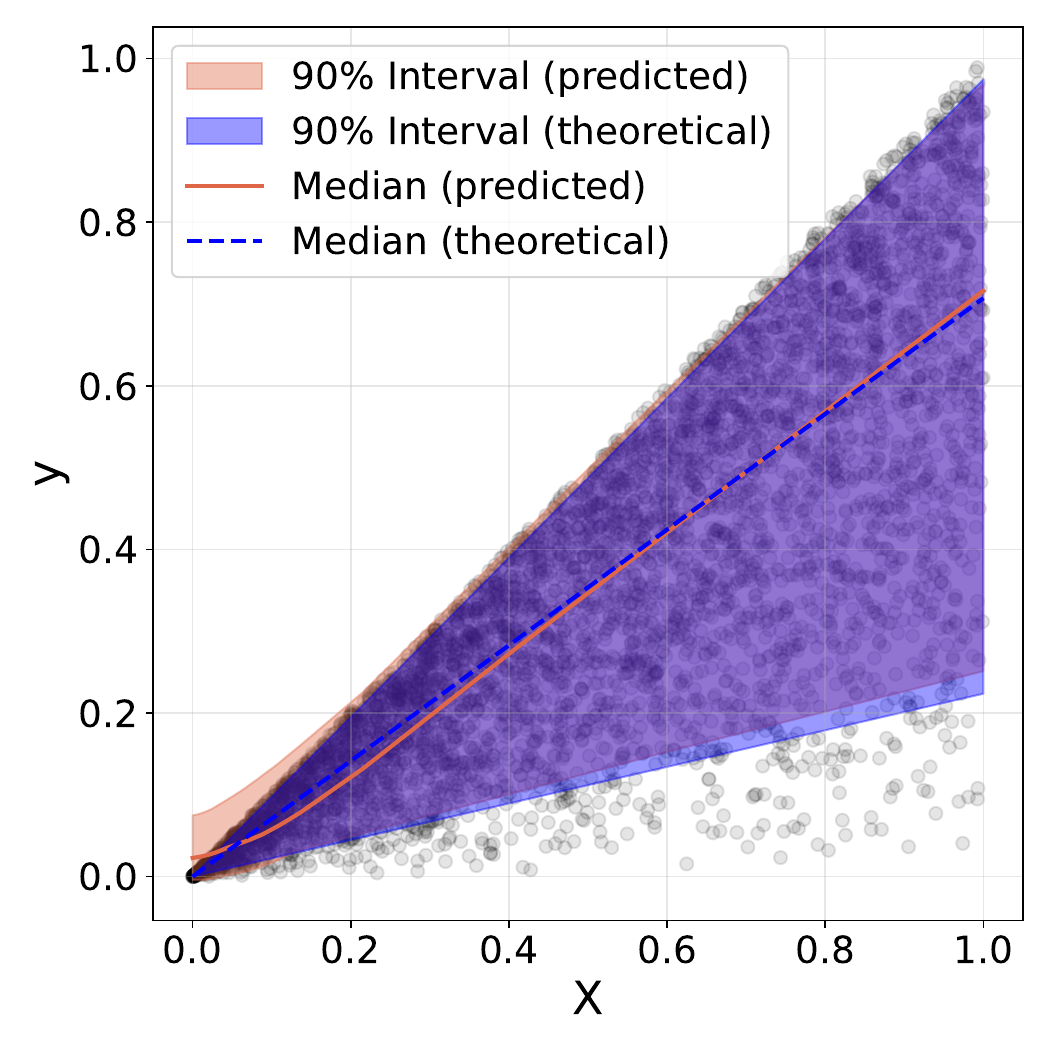}
\caption{Neural Network}
\label{fig:qrnn}
\end{subfigure}
\caption{Quantile regression results for the synthetic data. (a) Quantile regression forest and (b) quantile regression neural network show the estimated and theoretical 5\%, 50\%, and 95\% quantiles as well as the 90\% intervals.}
\label{fig:synthetic_quantiles}
\end{figure}

To generate the training data, we draw $n = 10000$ independent, univariate predictor samples $X_i$ from the uniform distribution on the interval $[0, 1]$. The response variable is then sampled i.i.d. according to:
\begin{equation*}
y \sim \text{Triangular}(0, x, x),
\end{equation*}
where $\text{Triangular}(0, x, x)$ is the Triangular distribution with lower limit $0$, upper limit $x$, and mode $x$. The conditional density is:
\begin{align*}
    f(y|x) = \frac{2y}{x^2} \mathds{1}\{y\in (0,x)\}. 
\end{align*}

Figure \ref{fig:synthetic_quantiles} presents the estimated quantiles alongside the theoretical quantile functions. 
In what follows, we derive the analytical forms of the prediction sets for CQR, CHR, and CTI based on the theoretical quantile functions.

\begin{figure}[t!]
\centering
\begin{subfigure}[b]{0.5\textwidth}
\centering
\includegraphics[width=\textwidth]{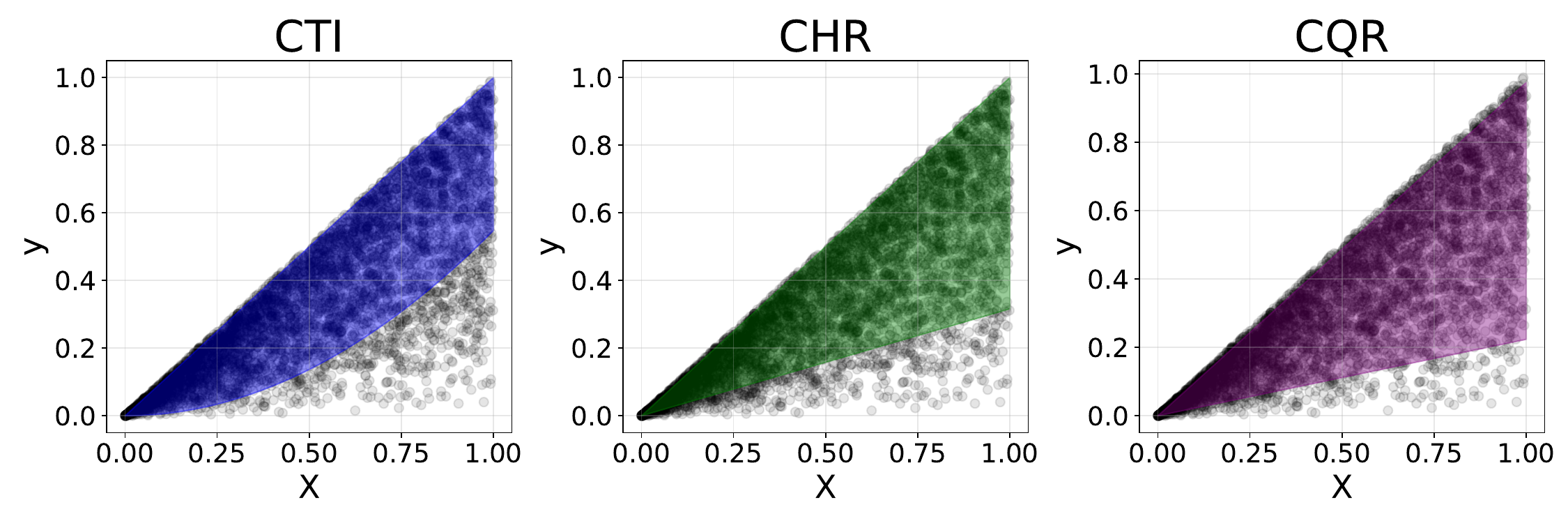}
\caption{Theoretical prediction set at $\alpha=0.1$}
\label{fig:theoretical_sets}
\end{subfigure}
\vfill
\begin{subfigure}[b]{0.45\textwidth}
\centering
\includegraphics[width=\textwidth]{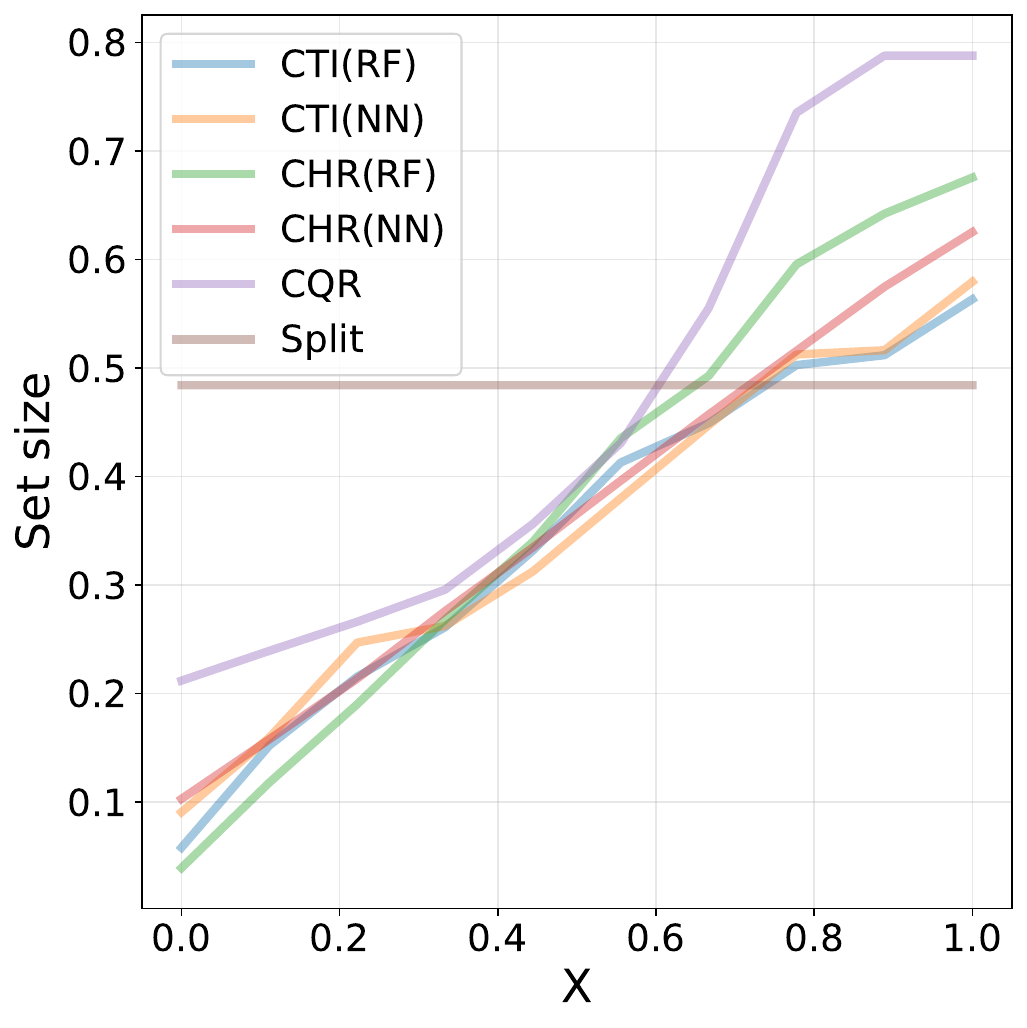}
\caption{Prediction set sizes at $\alpha=0.1$}
\label{fig:size_comparison}
\end{subfigure}
\caption{Prediction set sizes for the synthetic data at $\alpha=0.1$. (a) Theoretical prediction set for different conformal methods. The expected set sizes for CQR, CHR, and CTI are 0.376, 0.342, and 0.317, respectively. (b) Prediction set sizes as a function of $x$ using the estimated quantile functions (RF for random forests and NN for neural network). CTI achieves the smallest set size while maintaining guaranteed coverage.}
\label{fig:synthetic_results}
\end{figure}

Figure \ref{fig:synthetic_results} illustrates the theoretical and estimated prediction set sizes as a function of $x$ at $\alpha=0.1$. The theoretical expected set sizes for CQR, CHR, and CTI are 0.376, 0.342, and 0.317, respectively. The results demonstrate that CTI achieves the smallest set size while maintaining guaranteed coverage.

\begin{table*}[t!]
\centering
\begin{adjustbox}{max width=0.9\textwidth}
\begin{tabular}{lccccccc}
\toprule
Dataset & Metric & CTI(RF) & CTI(NN) & CHR(RF) & CHR(NN) & CQR & Split \\
\midrule
\multirow{2}{*}{synthetic} & Coverage & 0.900 (0.006) & 0.900 (0.007) & 0.903 (0.007) & 0.901 (0.006) & 0.903 (0.008) & 0.903 (0.007) \\
& Size & \textbf{0.345 (0.005)} & 0.369 (0.015) & 0.375 (0.006) & 0.370 (0.014) & 0.440 (0.022) & 0.482 (0.016) \\
\midrule
\multirow{2}{*}{bike} & Coverage & 0.898 (0.007) & 0.899 (0.007) & 0.898 (0.010) & 0.900 (0.006) & 0.906 (0.009) & 0.899 (0.008) \\
& Size & 1.032 (0.029) & \textbf{0.720 (0.028)} & 1.124 (0.028) & 0.758 (0.047) & 1.599 (0.054) & 1.345 (0.053) \\
\midrule
\multirow{2}{*}{bio} & Coverage & 0.900 (0.004) & 0.902 (0.004) & 0.899 (0.005) & 0.900 (0.004) & 0.900 (0.003) & 0.901 (0.004) \\
& Size & \textbf{1.295 (0.018)} & 1.474 (0.030) & 1.450 (0.023) & 1.576 (0.012) & 2.005 (0.016) & 1.961 (0.039) \\
\midrule
\multirow{2}{*}{blog} & Coverage & 0.910 (0.002) & 0.900 (0.004) & 0.902 (0.004) & 0.902 (0.003) & 0.940 (0.009) & 0.910 (0.006) \\
& Size & \textbf{0.709 (0.031)} & 1.003 (0.024) & 1.567 (0.074) & 1.737 (0.154) & 3.259 (0.327) & 1.453 (0.113) \\
\midrule
\multirow{2}{*}{community} & Coverage & 0.909 (0.018) & 0.908 (0.021) & 0.903 (0.015) & 0.905 (0.021) & 0.889 (0.024) & 0.902 (0.024) \\
& Size & 1.611 (0.088) & \textbf{1.275 (0.095)} & 1.637 (0.096) & 1.588 (0.100) & 1.680 (0.078) & 2.132 (0.188) \\
\midrule
\multirow{2}{*}{concrete} & Coverage & 0.908 (0.024) & 0.900 (0.031) & 0.899 (0.022) & 0.900 (0.023) & 0.901 (0.024) & 0.896 (0.021) \\
& Size & 0.967 (0.035) & \textbf{0.473 (0.050)} & 0.933 (0.041) & 0.505 (0.144) & 0.692 (0.051) & 0.619 (0.029) \\
\midrule
\multirow{2}{*}{facebook1} & Coverage & 0.909 (0.003) & 0.899 (0.003) & 0.901 (0.004) & 0.900 (0.004) & 0.945 (0.009) & 0.903 (0.002) \\
& Size & \textbf{0.766 (0.033)} & 0.780 (0.023) & 1.595 (0.088) & 1.379 (0.086) & 2.627 (0.329) & 2.252 (0.208) \\
\midrule
\multirow{2}{*}{facebook2} & Coverage & 0.911 (0.002) & 0.900 (0.001) & 0.899 (0.002) & 0.899 (0.002) & 0.943 (0.006) & 0.904 (0.002) \\
& Size & \textbf{0.735 (0.017)} & 0.773 (0.023) & 1.533 (0.053) & 1.382 (0.057) & 2.661 (0.272) & 2.100 (0.108) \\
\midrule
\multirow{2}{*}{homes} & Coverage & 0.900 (0.005) & 0.900 (0.006) & 0.899 (0.005) & 0.895 (0.007) & 0.898 (0.006) & 0.897 (0.005) \\
& Size & 0.640 (0.011) & \textbf{0.515 (0.008)} & 0.682 (0.012) & 0.535 (0.010) & 0.851 (0.052) & 0.825 (0.072) \\
\midrule
\multirow{2}{*}{meps19} & Coverage & 0.907 (0.008) & 0.902 (0.007) & 0.901 (0.007) & 0.902 (0.004) & 0.932 (0.007) & 0.902 (0.010) \\
& Size & \textbf{1.760 (0.087)} & 1.795 (0.061) & 2.388 (0.195) & 2.602 (0.128) & 2.923 (0.170) & 3.092 (0.377) \\
\midrule
\multirow{2}{*}{meps20} & Coverage & 0.904 (0.004) & 0.901 (0.007) & 0.901 (0.007) & 0.901 (0.006) & 0.927 (0.009) & 0.902 (0.005) \\
& Size & \textbf{1.883 (0.067)} & 1.921 (0.091) & 2.376 (0.105) & 2.594 (0.140) & 2.925 (0.193) & 3.154 (0.217) \\
\midrule
\multirow{2}{*}{meps21} & Coverage & 0.906 (0.005) & 0.900 (0.008) & 0.900 (0.006) & 0.898 (0.004) & 0.928 (0.007) & 0.905 (0.004) \\ 
& Size & \textbf{1.832 (0.089)} & 1.866 (0.076) & 2.510 (0.167) & 2.609 (0.145) & 2.971 (0.179) & 3.046 (0.199) \\
\midrule
\multirow{2}{*}{star} & Coverage & 0.903 (0.018) & 0.910 (0.017) & 0.907 (0.021) & 0.897 (0.018) & 0.901 (0.016) & 0.910 (0.024) \\
& Size & 0.186 (0.006) & 0.197 (0.009) & 0.182 (0.005) & 0.204 (0.009) & \textbf{0.181 (0.005)} & 0.181 (0.008) \\
\bottomrule
\end{tabular}
\end{adjustbox}
\caption{The coverage and size results for various methods are presented in the table. Our proposed Conformalized Thresholded Interval (CTI) method, which utilizes quantile regression based on either random forest (RF) or neural network (NN) models, demonstrates superior performance compared to other methods on most datasets.}
\label{tab:results}
\end{table*}

\subsection{Real Data}
Following the methodology outlined in \cite{sesia2020comparison}, we rescale the response $Y$ by the mean absolute value. We randomly allocate $20\%$ of the samples for testing, and from the remaining data, we utilize $70\%$ for training the quantile regression model and $30\%$ for calibration. This split has been validated in \cite{sesia2020comparison}. Wwe repeat all experiments 10 times, starting from the initial data splitting and the training procedure of quantile regression using both random forest (RF) and neural network (NN) models.

Both the CTI and CHR incorporate the quantile regression results from both RF and NN. For CQR, we use the NN results, as it is the choice in the original paper~\cite{romano2019conformalized}. We assess the performance of the generated prediction intervals in terms of coverage and efficiency. Since CTI and CHR utilize the same quantile regression results, the comparison between these two methods is fair with respect to the quality of the quantile regression. 

Table \ref{tab:results} presents the coverage probabilities and prediction set sizes for the various methods. CTI, which incorporates quantile regression based on either RF or NN models, outperforms other methods on most datasets. 

In a small dataset ``star", which has a relatively small sample size ($n=2161, d=39$), CQR outperforms CTI and CHR. The limited number of samples in this dataset may hinder the performance of the multi-output quantile regression model, as it requires sufficient data to accurately capture the underlying relationships between the features and the response variable. We also notice a similar trend in the relative performance comparison of CTI based on random forest and CTI based on neural network, as well as CHR based on random forest and CHR based on neural network. This suggests that the efficiency of the conformal prediction sets depends on the quality of the multi-output quantile regression. The choice of the underlying model plays a crucial role in the performance of the conformal prediction methods.

To evaluate the performance of CTI, we compare the lengths of response intervals (intervals containing the actual responses) with the lengths of all intervals generated by the multi-output quantile regression model across all datasets. The distributions of interval lengths for response intervals (blue histogram) and all intervals (red histogram) on the test set are shown in Figure \ref{fig:interval_length_comparison}. Although, theoretically, the difference in means between the two distributions should be zero, we observe that it is not. This discrepancy highlights the dependence of the prediction sets produced by CTI on the specific dataset and the underlying quantile regression model.

\begin{figure}[t!]
    \centering
    \begin{tabular}{@{}ccc@{}}
        \includegraphics[width=0.3\columnwidth]{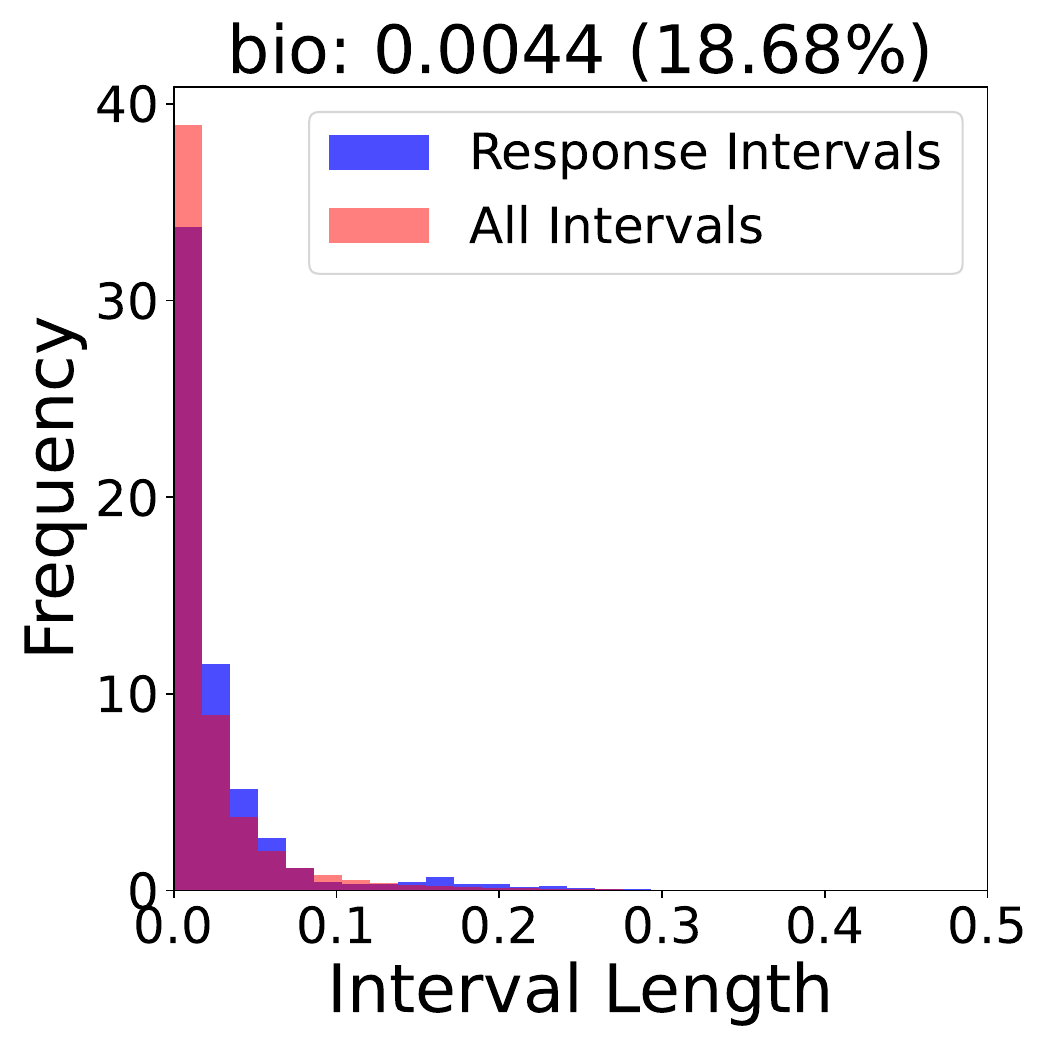} &
        \includegraphics[width=0.3\columnwidth]{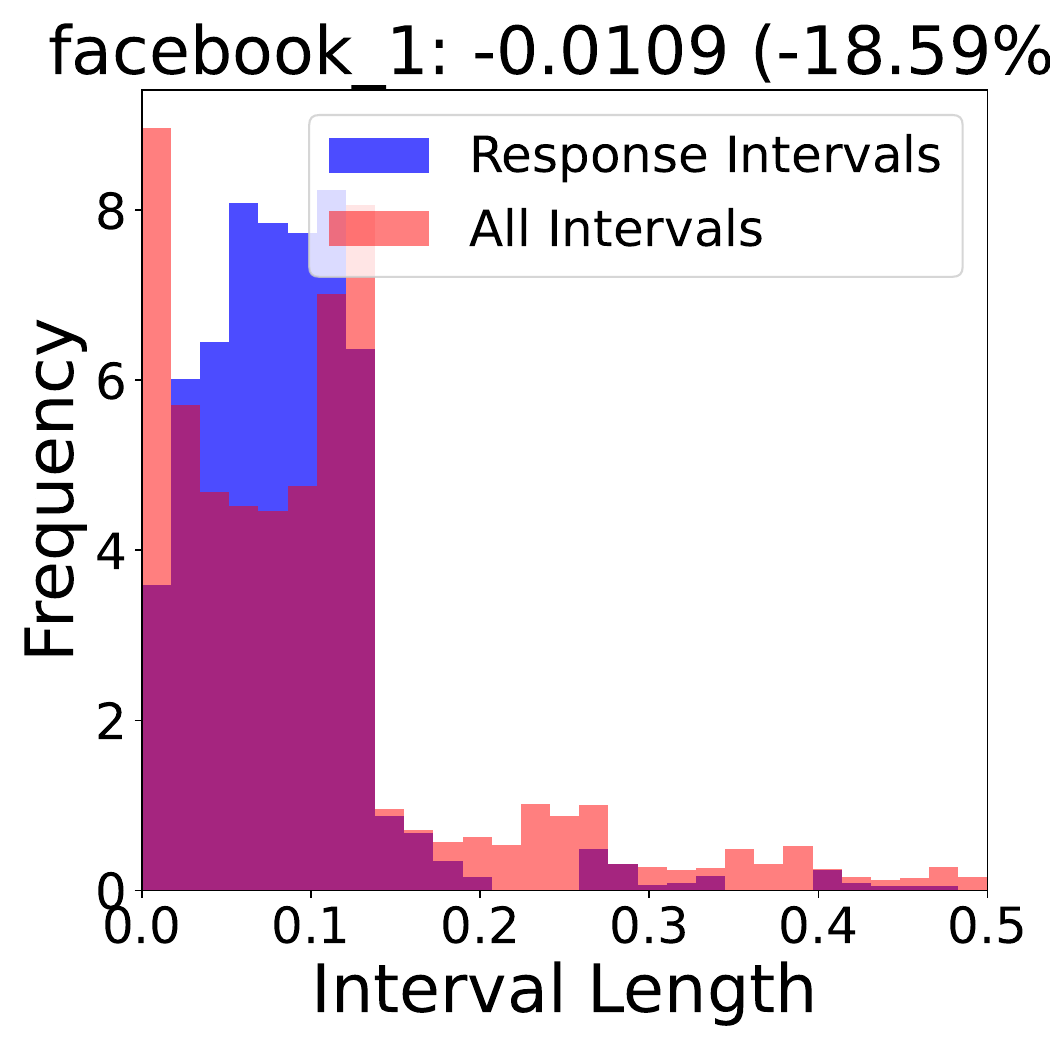} &
        \includegraphics[width=0.3\columnwidth]{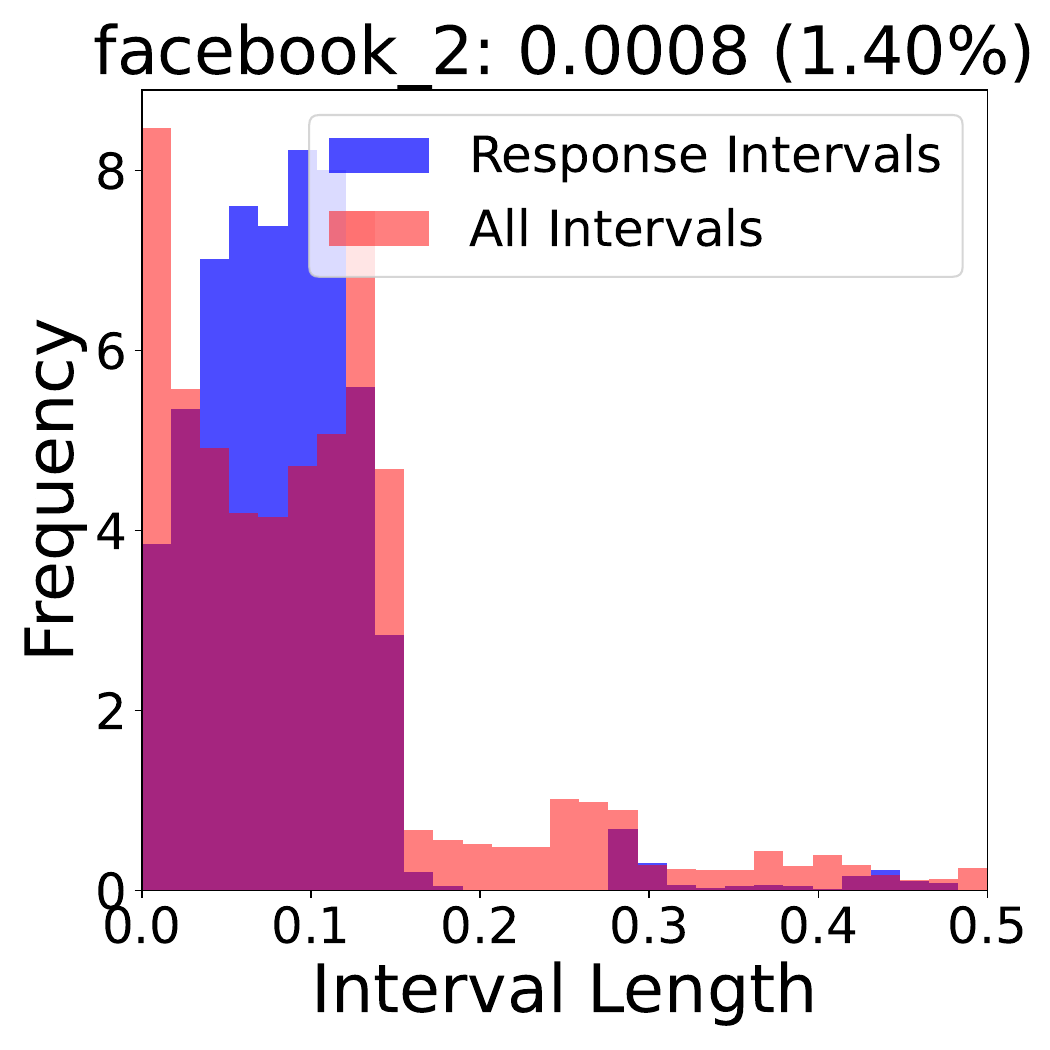} \\
        \includegraphics[width=0.3\columnwidth]{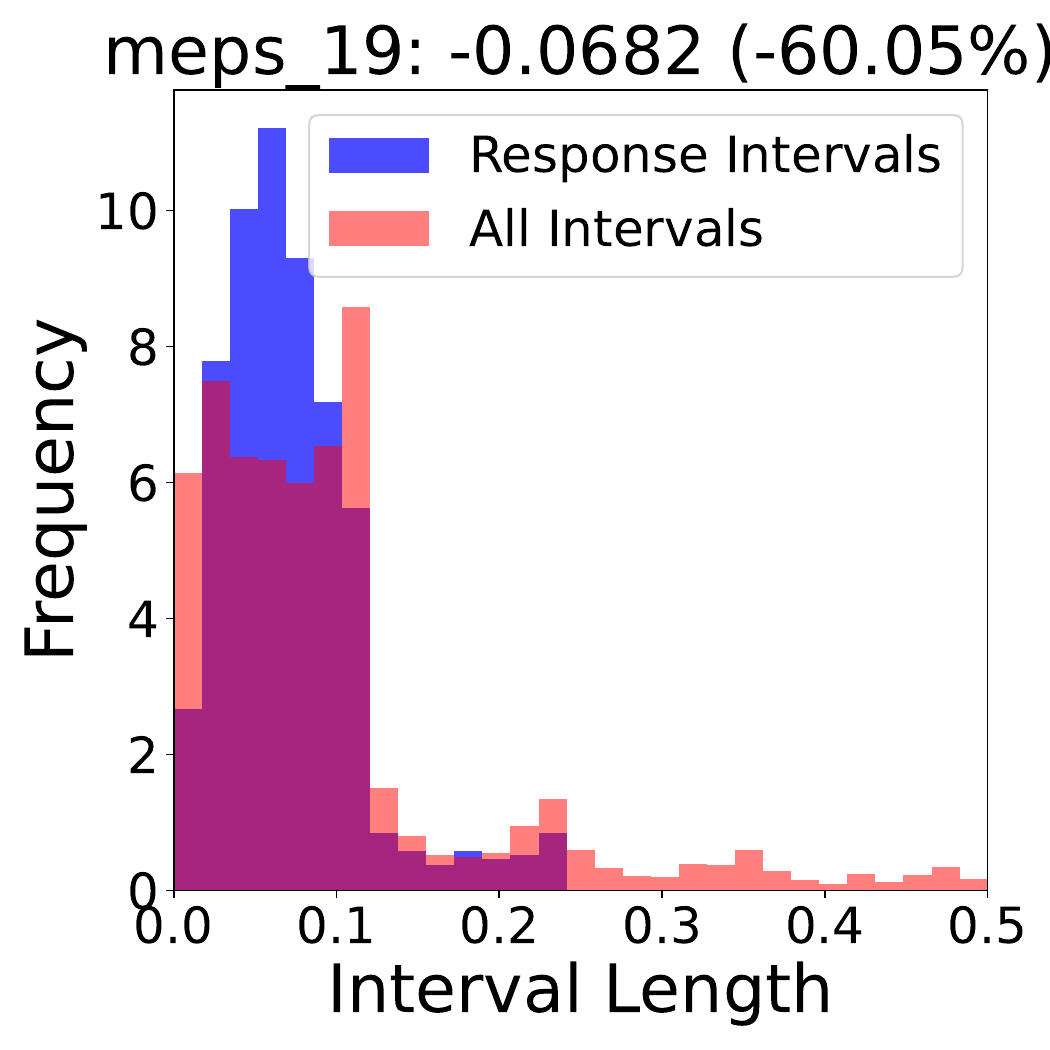} &
        \includegraphics[width=0.3\columnwidth]{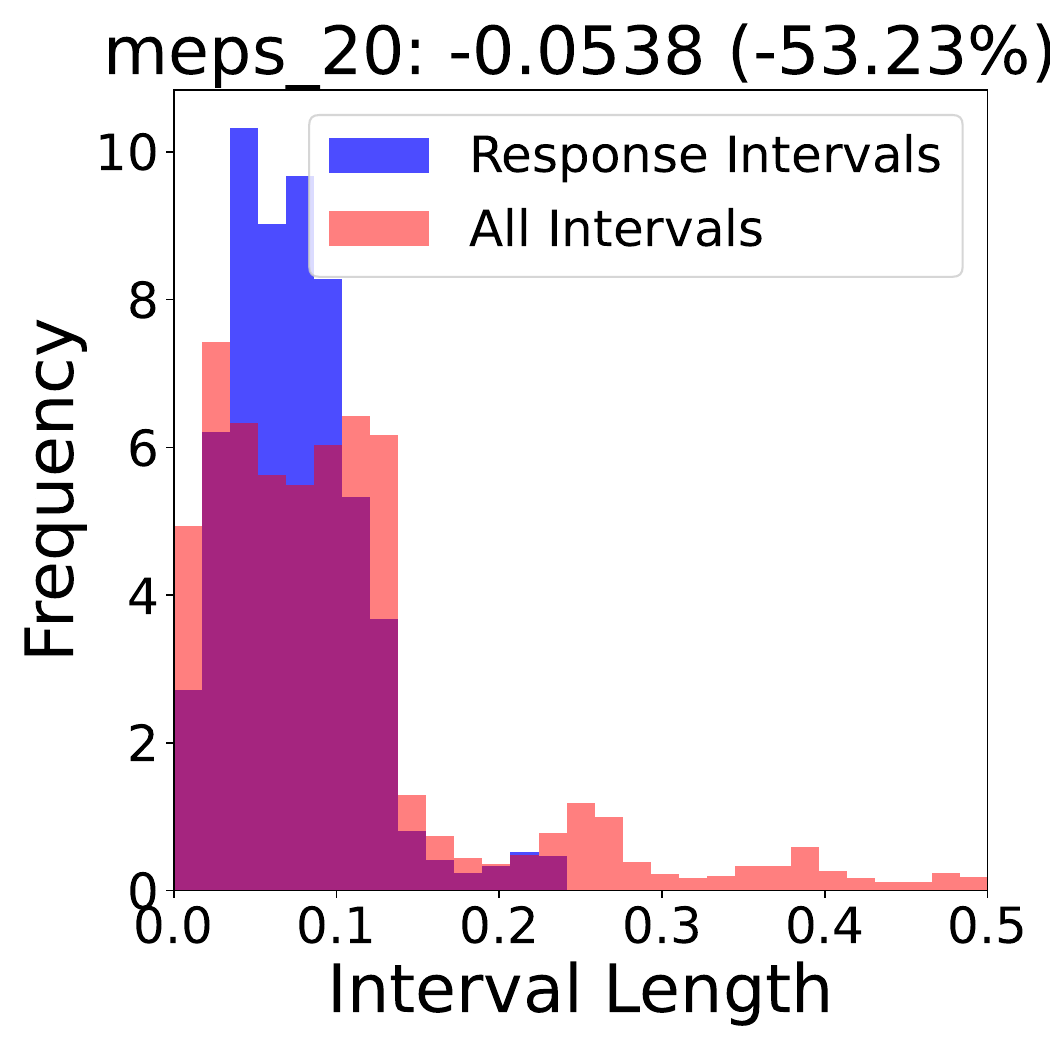} &
        \includegraphics[width=0.3\columnwidth]{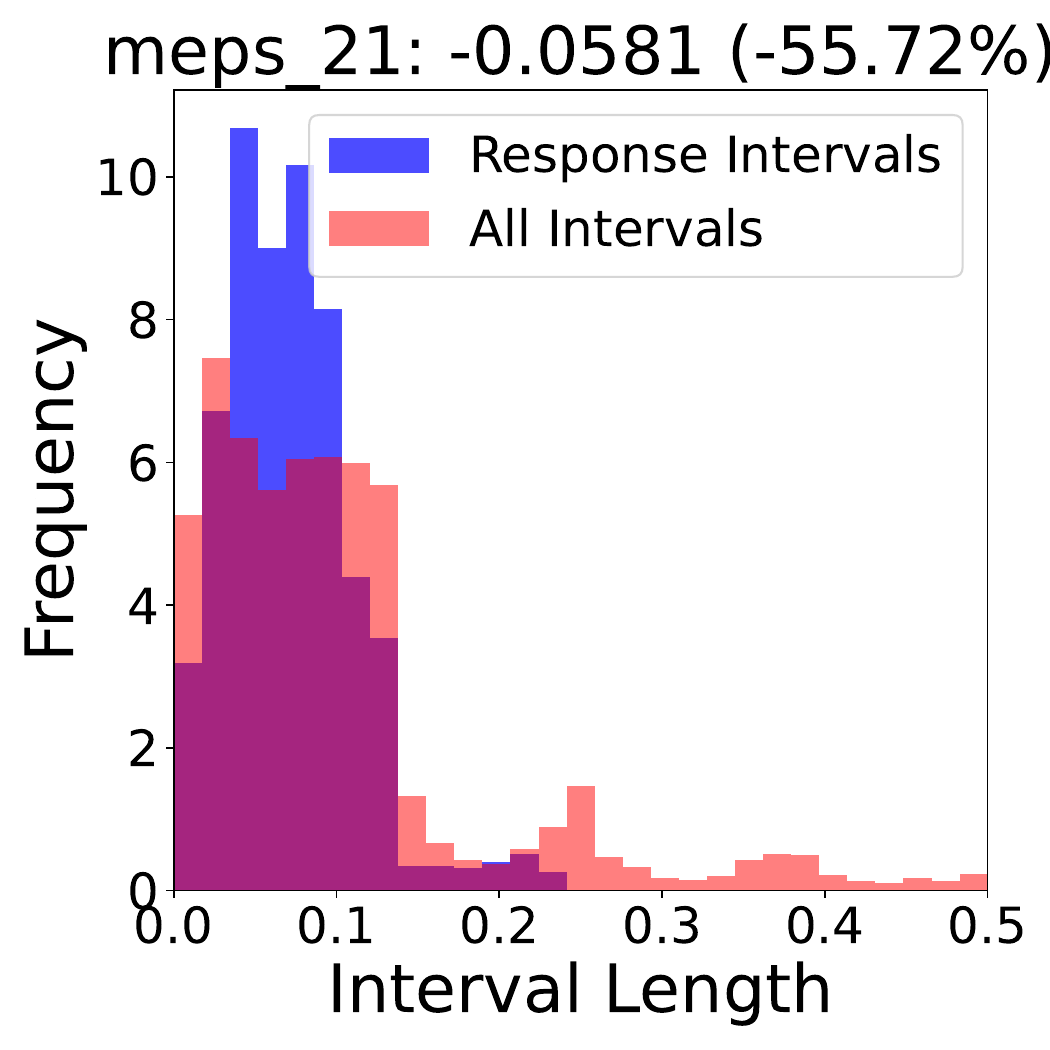} \\
        \includegraphics[width=0.3\columnwidth]{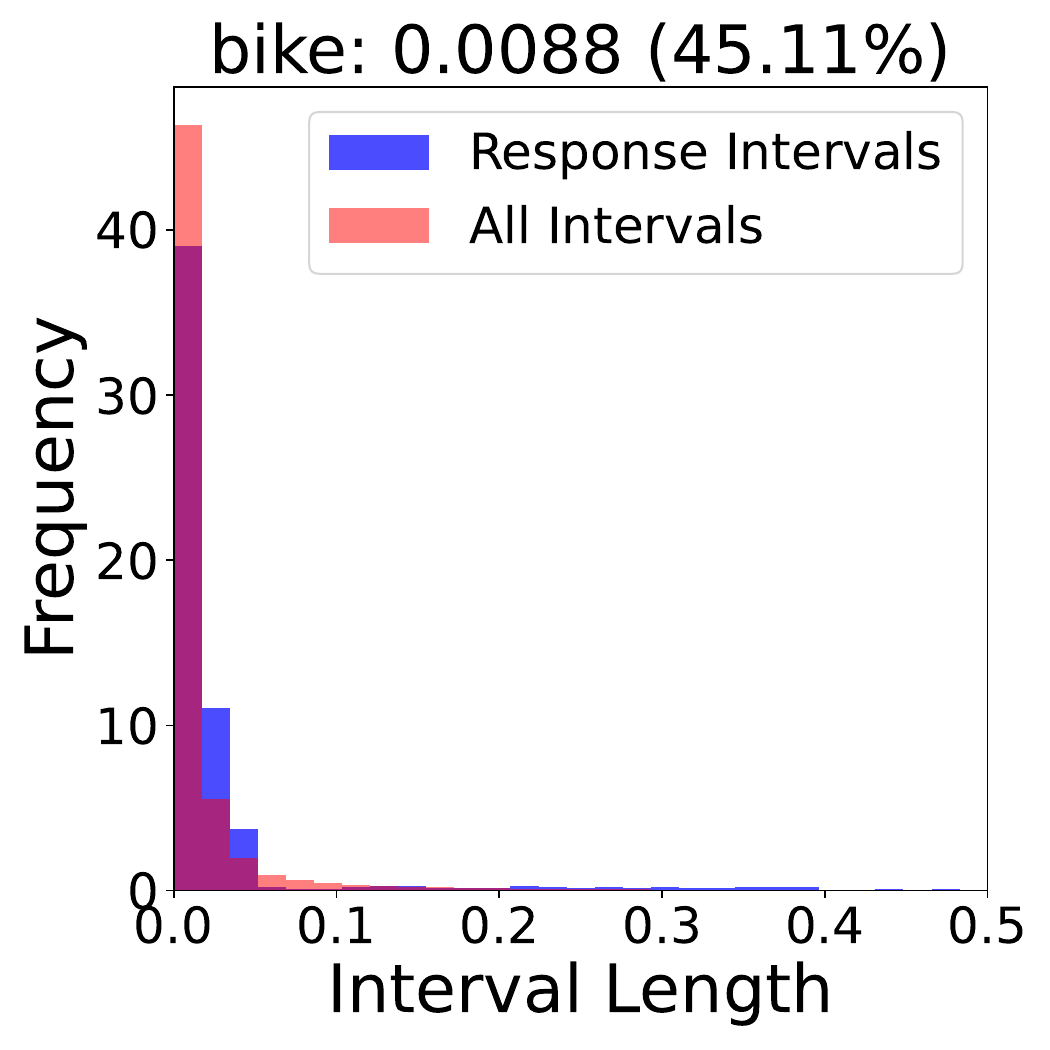} &
        \includegraphics[width=0.3\columnwidth]{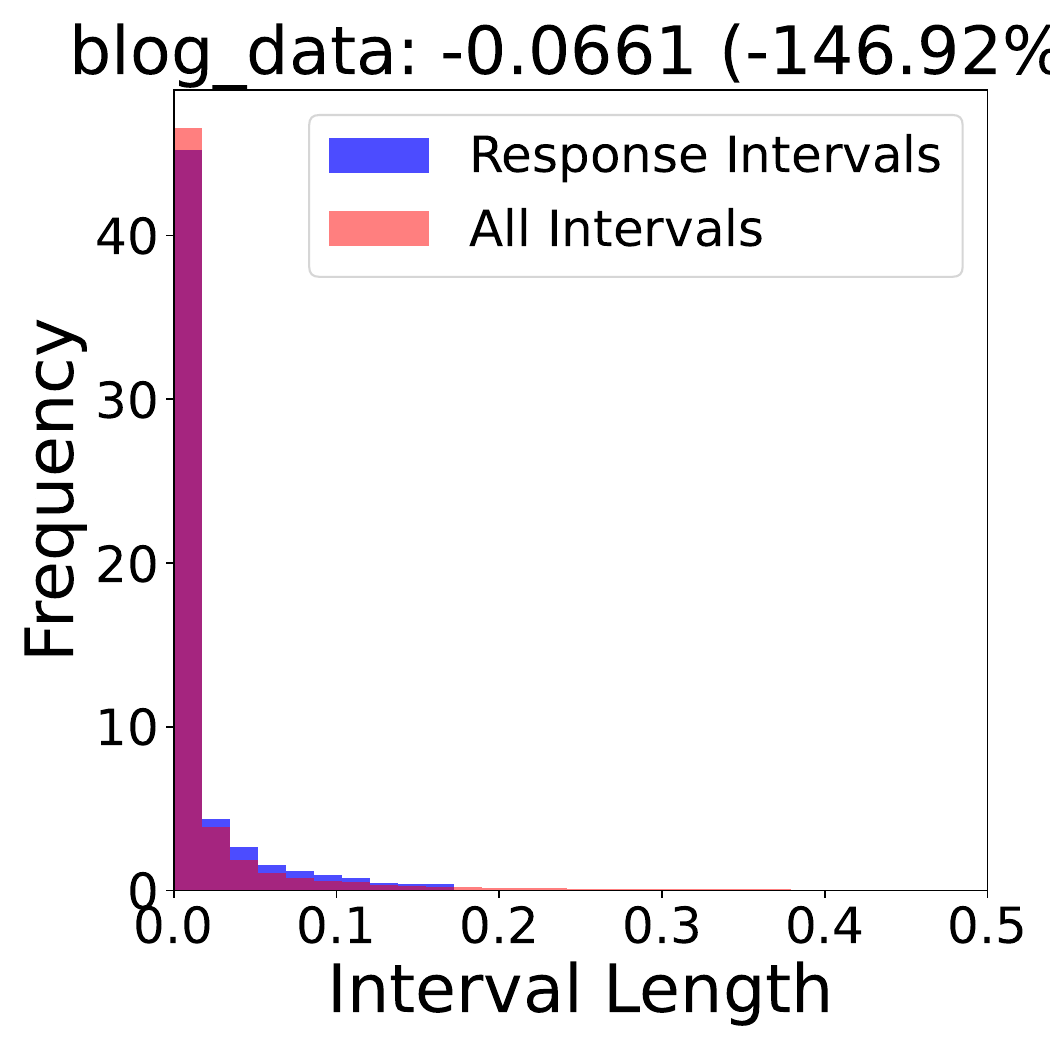} &
        \includegraphics[width=0.3\columnwidth]{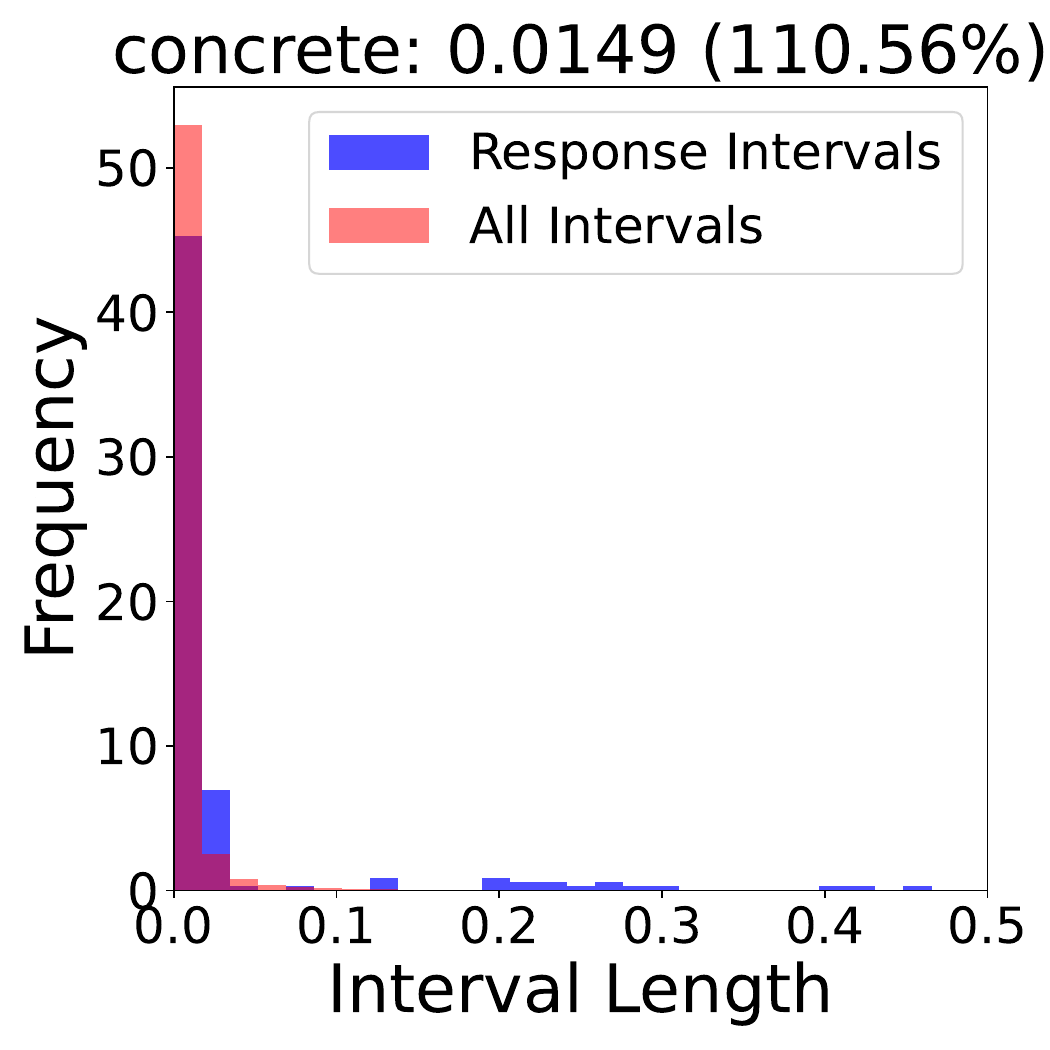} \\
        \includegraphics[width=0.3\columnwidth]{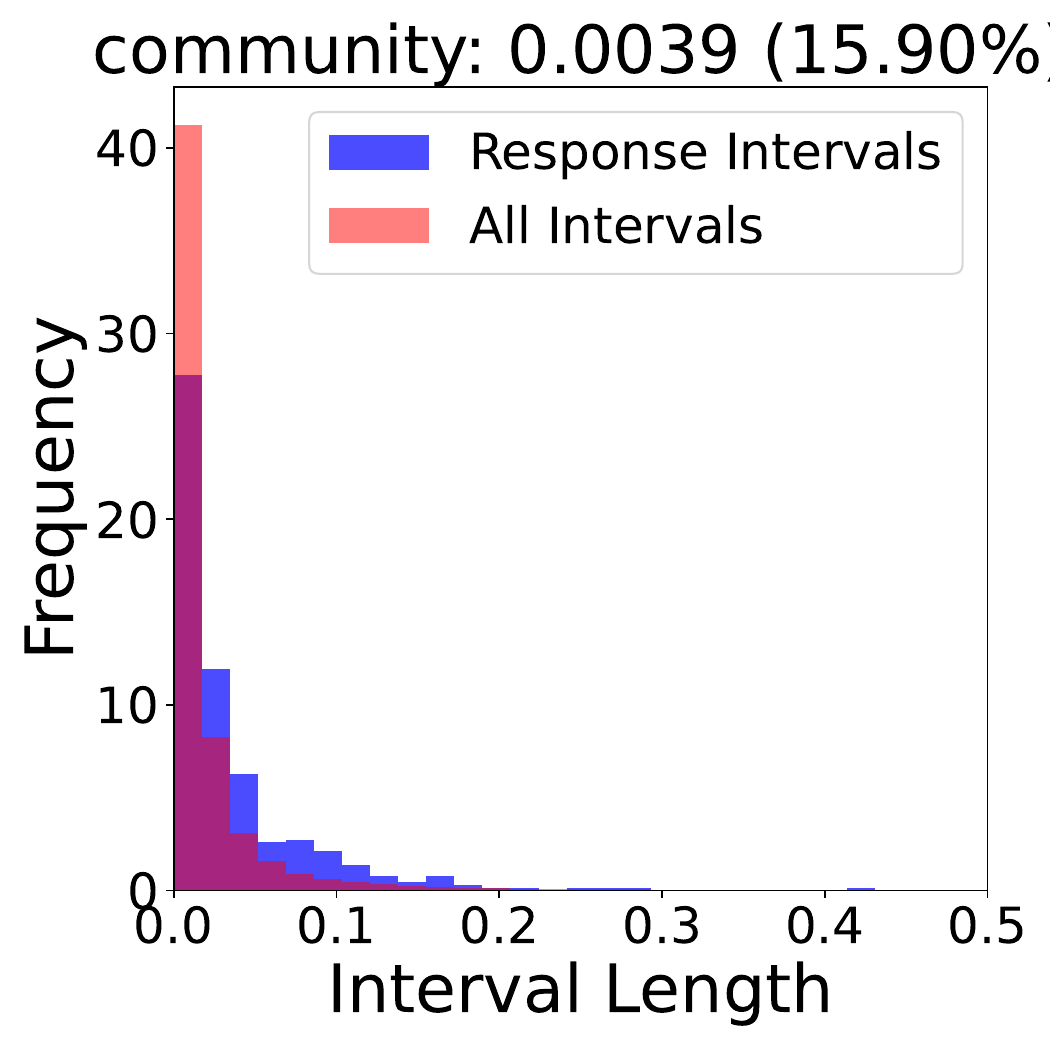} &
        \includegraphics[width=0.3\columnwidth]{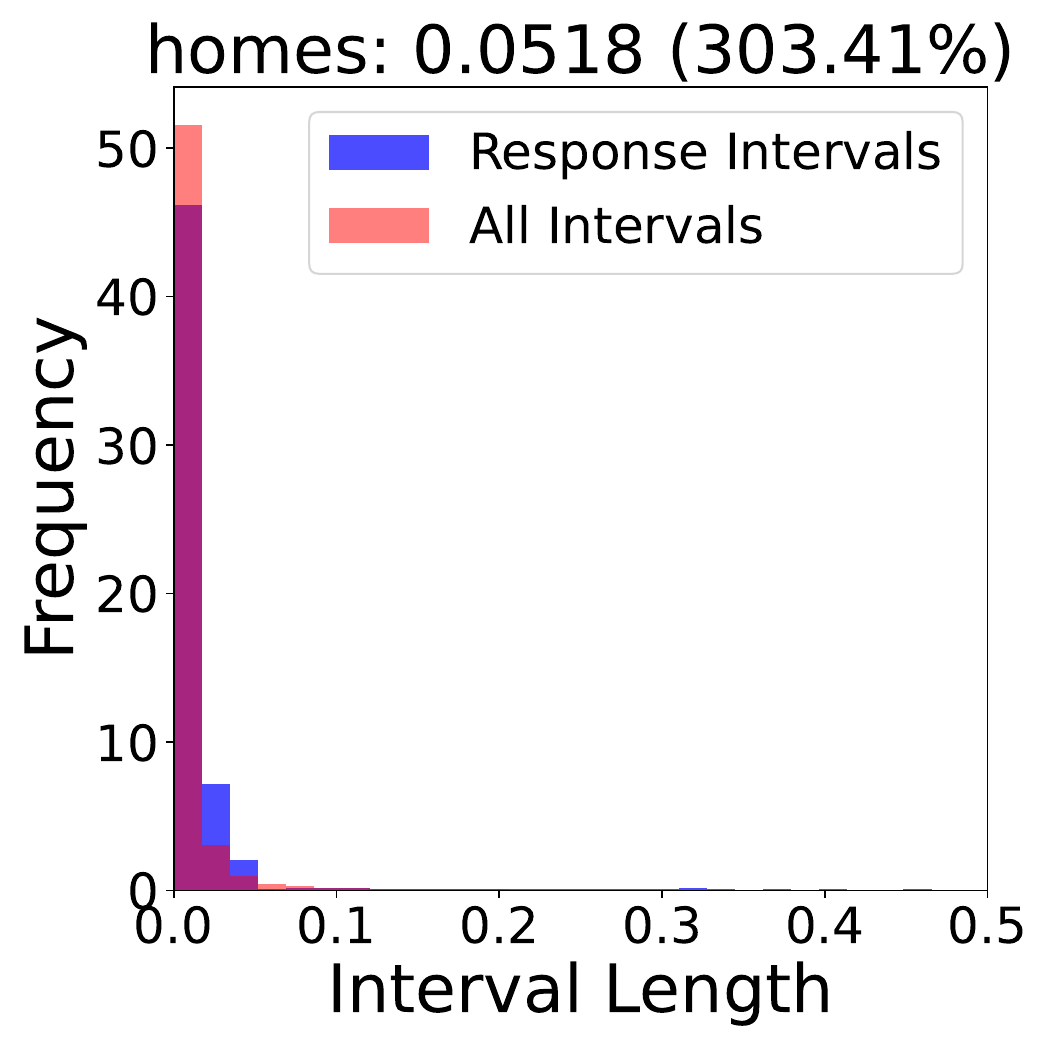} &
        \includegraphics[width=0.3\columnwidth]{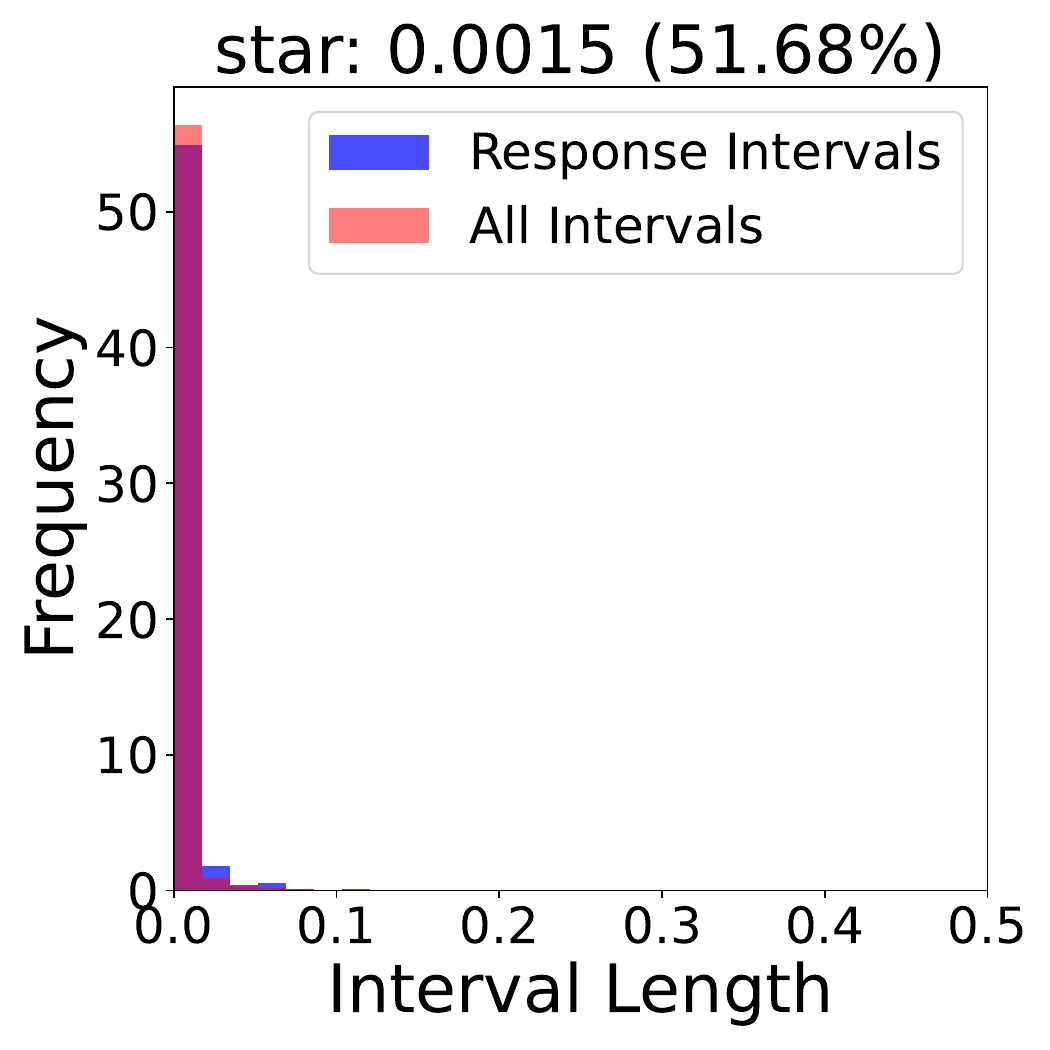}
    \end{tabular}
    \caption{Comparison of interval lengths across datasets: Each subfigure shows the distribution of interval lengths for a specific dataset. The blue histogram represents the intervals containing the actual responses, while the red histogram shows all intervals from the multi-output quantile regression model on the test set.}
    \label{fig:interval_length_comparison}
\end{figure}

\section{Conclusion}\label{sec:conclusion}

Conformal Thresholded Intervals (CTI) is a groundbreaking conformal prediction method specifically designed for regression tasks. It aims to construct the smallest possible prediction set while maintaining a guaranteed coverage level. The key innovation behind CTI lies in its clever utilization of multi-output quantile regression and the strategic thresholding of the estimated conditional interquantile intervals based on their length. By doing so, CTI effectively adapts to the local density of the data, enabling it to generate general prediction sets that are not limited to intervals. The ability of CTI to create prediction sets that are not restricted to intervals sets it apart from other conformal prediction methods. This flexibility allows CTI to better capture the intricate patterns and structures present in the data, ultimately leading to more precise and informative prediction sets. However, it is important to note that the current implementation of CTI does not guarantee that the resulting prediction sets will always be intervals. This limitation opens up an exciting avenue for future research, where the goal would be to develop a variant of CTI that can produce interval-based prediction sets while still preserving its adaptivity and efficiency.

\section{Acknowledgments}
The work described in this paper was partially supported by a grant from City University of Hong Kong (Project No.9610639).

\appendix

\section{Proof of Theorem \ref{thm:coverage}}

The score function $s(X_i, Y_i)$ for $i\in \Ical\cup\Itest$ are also exchangeable. For any $(X,Y)$ in the test set, the rank of $s(X, Y)$ is smaller than $t$ defined in~\eqref{eq:threshold} with probability $\frac{\lceil (1+|\Ical|)(1-\alpha)\rceil}{1+|\Ical|}\ge 1-\alpha$.

\section{Proof of Proposition \ref{prop:consistency}}

Let \( t' = 1/(K t) \) in the assumption. Then, by the definition of \( t \),
\begin{align*}
&\sum_{i \in \Ical} \mathbf{1}\big\{\mu(I_{k(y_i)}(x_i)) > \frac{1}{K t'}\big\} \\
=& |\Ical| - \lceil (1 + |\mathcal{I}_{\text{cal}}|)(1 - \alpha) \rceil =: |\mathcal{I}_{\text{cal}}| \alpha'.
\end{align*}

Under the assumption, \( F_{f(Y|X)}(t') \geq \alpha' - \epsilon \).

Note that instead of imposing direct assumptions on quantile accuracy, we derive the estimation error under data-specific conditions. Applying Theorem 5 from \cite{takeuchi2006nonparametric}, 
\[
\mathbb{P}\left( |\hat{q} - q| \leq \epsilon \right) \geq 1 - \delta,
\]
given appropriate data assumptions. Here, \( \epsilon \) is as defined above. This approach substantiates the assumptions in Proposition \ref{prop:consistency} by linking them to measurable data properties.

\section{Proof of Theorem \ref{thm:set}}

     In the construction of $\mathcal C_{1-\alpha}(x)$, we consider the intervals with length$\mu(I_k(x))\le t$. Given the assumption in the theorem, and combine it with Lemma~\ref{lem:density:bound}, we have
\begin{align*}
    \min_{y\in I_k} f(y|x) \ge \frac{1-\delta_k(x)}{Kt} - \frac {L(x)t}2=:t'.
\end{align*}
By the construction of $\mathcal C_{1-\alpha}(x)$, we have the conclusion of the theorem. 

\section{Gap Between Theorem \ref{thm:coverage} and Algorithm \ref{alg:cti}}

Using Theorem 5 from \cite{takeuchi2006nonparametric}, we bound the quantile estimation error. Let $\hat{q}$ be the estimated quantile from samples, then with probability $1 - \delta$,
\[
|\hat{q} - q| \leq \epsilon,
\]
where $\epsilon = 2\mathcal{R}_n(\mathcal{F}) + \sqrt{\frac{8 \log \frac{2}{\delta}}{n}}$, with $\mathcal{R}_n(\mathcal{F})$ representing the Rademacher Complexity of the function class of the conditional quantile functions.

As the sample size $n \to \infty$, $\epsilon \to 0$. This bound ensures that the estimation error is controlled, thereby validating the claim in Theorem \ref{thm:coverage} despite potential estimation inaccuracies.

\section{Number of Interquantile Intervals $K$}

The hyperparameter $K$ influences the performance of the proposed algorithm CTI by balancing the granularity of quantile estimates and the complexity of the prediction sets. We provide bounds for $K$ from two directions based on Theorem 5 from \cite{takeuchi2006nonparametric}.

Assume that $\mathcal{H}$ is a Reproducing Kernel Hilbert Space (RKHS) with a radial basis function kernel $k$ for which $k(x,x) = 1$. Moreover, assume that for all $f \in \mathcal{F}$, we have $\|f\|_{\mathcal{H}} \leq \mathcal{C}$. From \cite{mendelson2003few}, it follows that
\[
\mathcal{R}_n(\mathcal{F}) \leq \frac{2\mathcal{C}}{\sqrt{n}}.
\]

Using Theorem 5 from \cite{takeuchi2006nonparametric}, the quantile estimation error bound is 
\[
|\hat{q} - q| \leq \epsilon,
\]
where
\[
\epsilon = 2\mathcal{R}_n(\mathcal{F}) + \sqrt{\frac{8 \log (2n)}{n}} = \mathcal{O}\left(\sqrt{\frac{\log n}{n}}\right).
\]
As the sample size $n \to \infty$, $\epsilon \to 0$ since $\sqrt{\frac{\log n}{n}} \to 0$.

Additionally, to ensure the prediction sets remain valid, it is necessary that 
\[
K \epsilon = \frac{K}{\sqrt{n}} \to 0.
\]
This condition suggests that 
\[
K = \mathcal{O}(\sqrt{n}).
\]

Based on this analysis, we recommend choosing $K$ in the order of 
\[
K = \mathcal{O}(\sqrt{n}).
\]

For CHR \cite{sesia2021conformal}, it is also suggested selecting $K$ in the order of
\[
K = \mathcal{O}(n^{\frac{1}{3}}).
\]

In our experiments, we used $K = 100$, which aligns with the recommended order and demonstrated consistent results across all the datasets.

\section{Ancillary Lemmas}

\begin{lemma}\label{lem:density:bound}
    For a Lipschitz function $f:\mathbb R\to\mathbb R$ with Lipschitz constant $L$, if $\int_a^bf(x)d\mu(x)=c$, then $f(x)\in [\frac{c}{b-a}-\frac{L(b-a)}{2}, \frac{c}{b-a}+\frac{L(b-a)}{2}]$ for all $x \in [a, b]$. 
\end{lemma}

\begin{proof}
    Suppose $f(x')= d$ for some $x'\in[a,b]$. Then, by the Lipschitz condition, we have:
    \begin{align*}
        |f(x) - f(x')| \leq L|x - x'| \implies f(x) \leq d + L|x - x'|
    \end{align*}
    for all $x \in [a, b]$. Integrating both sides over $[a, b]$, we get:
    \begin{align*}
        \int_a^b f(x) d\mu(x) &\leq \int_a^b (d + L|x-x'|) d\mu(x) \\
        &= d(b-a) + L\int_a^b |x-x'| d\mu(x) \\
        &\leq d(b-a) + \frac{1}{2} L(b-a)^2,
    \end{align*}
    where the last inequality follows from the fact that $\int_a^b |x-x'| d\mu(x) \leq \frac{1}{2}(b-a)^2$. Since $\int_a^b f(x) d\mu(x) = c$, we have:
    \begin{align*}
        c &\leq d(b-a) + \frac{1}{2} L(b-a)^2 \\
        \implies d &\geq \frac{1}{b-a} \left(c - \frac{1}{2} L(b-a)^2\right) \\
        &= \frac{c}{b-a} - \frac{L(b-a)}{2}.
    \end{align*}
    Since $x'$ was arbitrary, this lower bound holds for all $x \in [a, b]$. The upper bound can be proved analogously.
\end{proof}

\begin{lemma}\label{lem:threshold}
Let $\mathcal C(x) = \{y: f(y|x)\ge t'\}$. Then the smallest $t'$ satisfying
\begin{align*}
    \int_{\mathcal X} \int_{\mathcal C(x)} f(y|x) d\mu(y)dP(x)\ge 1-\alpha
\end{align*}
is given by
\begin{align*}
    t' = \inf\{ t\in\mathbb R: \mathbb P(f(Y|X)\ge t)\ge 1-\alpha \}.
\end{align*}
\end{lemma}
\begin{proof}
    By direct calculation.
    \begin{align*}
        \mathbb P(f(Y|X)\ge t)
        &=\int_{\mathcal X} \int_{\mathcal Y} \mathbf 1\{f(y|x)\ge t\} dP(y|x) dP(x)\\
        &=\int_{\mathcal X} \int_{\{f(y|x)\ge t\}} f(y|x) d\mu(y) dP(x) 
    \end{align*}
    The proof is complete. 
\end{proof}

\section{Calculation of Simulation Study}

For $\tau\in[0,1]$, let $q_\tau(x)$ be the $\tau$-th quanile of $Y|X=x$, then
\begin{align*}
    \tau = \int_{0}^{q_\tau(x)} \frac{2y}{x^2}\ dy = \frac{y^2}{x^2} |_{y=0}^{q_\tau(x)} = \frac{[q_\tau(x)]^2}{x^2}.
\end{align*}
Hence, $q_{\tau}(x) = x\sqrt{\tau}$. 
We use quantile regression forest and quantile regression neural network to estimate the quantile functions $\hat{q}$. 

\noindent
\textbf{Conformal Quantile Regression (CQR).} CQR uses
\begin{align*}
    [q_{\alpha/2}(x), q_{1-\alpha/2}(x)] 
\end{align*}
as the conditional prediction set. Its expected size is
\begin{align*}
    \int_0^1 q_{1-\alpha/2}(x) - q_{\alpha/2}(x) dx & = \left(\sqrt{1-\frac\alpha 2} - \sqrt {\frac\alpha 2}\right)\int_0^1 x dx \\ 
    &= \frac 12\left(\sqrt{1-\frac\alpha 2} - \sqrt {\frac\alpha 2}\right).
\end{align*}

\noindent
\textbf{Conformal Histogram Regression (CHR).} The ideal case of CHR uses
\begin{align*}
    [q_{\alpha}(x), q_{1}(x)] 
\end{align*}
as the conditional prediction set, which is the interval that accumulates the most density at each $x$, as shown in Figure \ref{fig:theoretical_sets}. Its expected size is
\begin{align*}
    \int_0^1 q_{1}(x) - q_{\alpha}(x) dx = \frac 12\left(1 - \sqrt {\alpha}\right).
\end{align*}

\noindent
\textbf{CTI.} CTI uses
\begin{align*}
    \{y: f(y|x)\ge t\}
\end{align*}
as the prediction set. If $t\le 2$, $f(y|x) = \frac{2y}{x^2} \mathbf 1\{(0,x)\}$ suggests that the level curve has the form $y = \frac 12 tx^2$ for $x\in(0,1)$. To determine the value of $t$, we have
\begin{align*}
    \alpha=\int_0^1\int_0^{tx^2/2} \frac {2y}{x^2}dydx = \int_0^1 \frac{t^2x^2}{4}dx = \frac{t^2}{12}.
\end{align*}
$t=\sqrt{12\alpha}$. The expected size is
\begin{align*}
    \int_0^1 x-\sqrt {3\alpha}x^2 dx = \frac 12 - \frac {\sqrt{3\alpha}}3. 
\end{align*}

\end{document}